\documentclass[11pt]{article}

\usepackage[final]{acl}

\usepackage{times}
\usepackage{latexsym}
\usepackage[T1]{fontenc}
\usepackage[utf8]{inputenc}
\usepackage{microtype}
\usepackage{inconsolata}
\usepackage{graphicx}
\usepackage{amsmath}
\usepackage{amssymb}
\usepackage{amsthm}
\usepackage{booktabs}
\usepackage{multirow}
\usepackage{array}
\usepackage{tabularx}
\usepackage{xspace}
\usepackage{algorithm}
\usepackage{algorithmic}
\usepackage{makecell}
\usepackage{array}
\usepackage[table]{xcolor}
\usepackage{colortbl}
\usepackage{booktabs}
\usepackage{caption}
\usepackage{cuted}
\usepackage{soul}
\soulregister\ref7

\usepackage{xcolor}
\usepackage[most]{tcolorbox}

\newtcolorbox{yellowblock}{
  breakable,
  colback=yellow!20,
  colframe=yellow!20,
  boxrule=0pt,
  left=2pt,right=2pt,top=2pt,bottom=2pt,
  arc=1pt
}

\newtcolorbox{blueblock}{
  breakable,
  colback=blue!8,
  colframe=blue!25,
  boxrule=0pt,
  left=2pt,right=2pt,top=2pt,bottom=2pt,
  arc=1pt
}

\definecolor{tblsection}{RGB}{242,242,242}
\definecolor{tblbase}{RGB}{255,255,255}
\definecolor{tblours}{RGB}{221,235,247}
\definecolor{tblts}{RGB}{226,240,217}
\definecolor{tbloursTS}{RGB}{224,238,230} 

\definecolor{tblHansHeader}{RGB}{191, 215, 234}
\definecolor{tblHansSection}{RGB}{221, 235, 247}
\definecolor{tblHansA}{RGB}{248, 251, 255}
\definecolor{tblHansB}{RGB}{238, 245, 252}

\definecolor{tblmcHeader}{RGB}{235,222,192}   
\definecolor{tblmcBodyA}{RGB}{252,249,241}    
\definecolor{tblmcBodyB}{RGB}{245,238,222}    

\definecolor{tblHeader}{RGB}{235,235,235}   
\definecolor{tblBlue}{RGB}{221,235,247}     
\definecolor{tblRed}{RGB}{248,220,220}      

\definecolor{tblGreenHeader}{RGB}{198,224,180}
\definecolor{tblGreenSection}{RGB}{226,239,218}
\definecolor{tblGreenA}{RGB}{248,252,246}
\definecolor{tblGreenB}{RGB}{239,247,235}

\definecolor{tblLavHeader}{RGB}{217,210,233}
\definecolor{tblLavSection}{RGB}{234,229,245}
\definecolor{tblLavA}{RGB}{250,249,253}
\definecolor{tblLavB}{RGB}{243,240,249}

\definecolor{tblHansHeader}{RGB}{235,235,235} 
\definecolor{tblHansA}{RGB}{236,231,246}      
\definecolor{tblHansB}{RGB}{255,237,213}      

\definecolor{tblHeurHeader}{RGB}{235,235,240} 
\definecolor{tblHeurTop}{RGB}{226,244,242}    
\definecolor{tblHeurBot}{RGB}{255,244,214}    

\definecolor{tblSPHeader}{RGB}{235,235,240} 
\definecolor{tblTauNine}{RGB}{232,245,233}  
\definecolor{tblTauEight}{RGB}{227,242,253} 
\definecolor{tblTauSeven}{RGB}{255,243,224} 
\definecolor{tblMNLI}{RGB}{244,252,248}     
\definecolor{tblSQuAD}{RGB}{255,246,246}    

\definecolor{tblAblHead}{RGB}{233,236,239}  
\definecolor{tblAblOrig}{RGB}{248,248,248}  
\definecolor{tblAblK}{RGB}{241,236,255}     
\definecolor{tblAblQ}{RGB}{232,248,252}     
\definecolor{tblAblV}{RGB}{239,252,232}     
\definecolor{tblAblQKV}{RGB}{255,244,224}   

\definecolor{tblLDHead}{RGB}{235,235,240}   
\definecolor{tblLDUn}{RGB}{242,238,255}     
\definecolor{tblLDUnTot}{RGB}{230,224,255}  
\definecolor{tblLDAns}{RGB}{236,252,244}    
\definecolor{tblLDAnsTot}{RGB}{220,244,232} 

\definecolor{tblAHead}{RGB}{230,236,245}
\definecolor{tblARowA}{RGB}{245,248,252}
\definecolor{tblARowB}{RGB}{235,242,250}

\definecolor{tblBHead}{RGB}{238,232,248}
\definecolor{tblBRowA}{RGB}{250,248,253}
\definecolor{tblBRowB}{RGB}{243,238,251}

\definecolor{tblNegHead}{RGB}{232,232,236}  
\definecolor{tblNegBase}{RGB}{248,248,248}  
\definecolor{tblNegUAT}{RGB}{232,248,252}   
\definecolor{tblNegTS}{RGB}{236,252,244}    

\definecolor{tblMergeHead}{RGB}{235,235,240} 
\definecolor{tblMergeA}{RGB}{235,242,250}    
\definecolor{tblMergeB}{RGB}{243,238,251}    

\definecolor{tblHansSection}{RGB}{240,240,240}

\definecolor{tblsection}{RGB}{242,242,242}

\definecolor{tbladdA}{RGB}{251,246,255} 
\definecolor{tbladdB}{RGB}{244,236,252} 

\definecolor{tblmcSection}{RGB}{255,247,235} 
\definecolor{tblmcA}{RGB}{255,252,245}       
\definecolor{tblmcB}{RGB}{247,238,220}       

\newcommand{\TS}{\textsc{TS}\xspace}
\newcommand{\ours}{\textsc{UAT-Lite}\xspace}

\soulregister\ours0
\soulregister\textbf7
\soulregister\emph7
\soulregister\texttt7
\soulregister\ref7
\soulregister\cite7

\usepackage[dvipsnames]{xcolor}
\usepackage{colortbl}
\usepackage{adjustbox}

\usepackage[shortlabels,inline]{enumitem}
\usepackage{tikz,lipsum}
\usepackage{ragged2e}
\usepackage[dvipsnames]{xcolor}
\usepackage{amsmath}
\usepackage{booktabs}
\usepackage{tabularray}
\usepackage{arydshln}
\usepackage{stmaryrd}
\usepackage{marvosym}
\usepackage{colortbl}
\usepackage{multicol}
\usepackage{multirow}
\usepackage{float}
\usepackage{amsfonts}
\usepackage{amssymb}
\usepackage{hyperref}

\usepackage{cleveref}
\crefname{section}{\S}{\S}
\crefname{table}{Table}{Tables}
\crefname{figure}{Fig.}{Figs.}
\crefname{algorithm}{Alg.}{}
\crefname{ALC@unique}{Line}{Lines}
\crefname{equation}{Eq.}{Eqs.}
\crefname{appendix}{App.}{Apps.}
\crefformat{section}{\S#2#1#3}

\usepackage{soul}

\usepackage{todonotes}

\newtheorem{proposition}{Proposition}

\title{\textsc{UAT-Lite}: Inference-Time Uncertainty-Aware Attention for Pretrained Transformers}

\author{
\textbf{Elias Hossain}$^{1}$\thanks{Corresponding author: \texttt{mdelias.hossain@ucf.edu}},
\textbf{Shubhashis Roy Dipta}$^{2}$,
\textbf{Subash Neupane}$^{3}$,
\textbf{Rajib Rana}$^{4}$,\\
\textbf{Ravid Shwartz-Ziv}$^{5}$,
\textbf{Ivan Garibay}$^{1}$,
\textbf{Niloofar Yousefi}$^{1}$
\\\\
$^{1}$University of Central Florida \quad
$^{2}$University of Maryland, Baltimore County\\
$^{3}$Meharry Medical College \quad
$^{4}$University of Southern Queensland \quad
$^{5}$New York University
\\\\
\texttt{\textbf{mdelias.hossain@ucf.edu}}
}

\begin{document}
\maketitle



\begin{abstract}
Neural NLP models are often miscalibrated and overconfident, assigning high confidence to incorrect predictions and failing to express uncertainty during internal evidence aggregation. This undermines selective prediction and high-stakes deployment.
Post-hoc calibration methods adjust output probabilities but leave internal computation unchanged, while ensemble and Bayesian approaches improve uncertainty at substantial training or storage cost.
We propose \ours, an inference-time framework that makes self-attention uncertainty-aware via Monte Carlo dropout in pretrained transformer classifiers.
Unlike output-level calibration (e.g., \TS), \ours injects epistemic uncertainty directly into attention, enabling uncertainty-aware routing during contextualization and token-level diagnostic signals beyond global logit rescaling.
Token-level epistemic uncertainty is estimated from stochastic forward passes and used to modulate self-attention during contextualization, without modifying pretrained weights or training objectives.
We additionally introduce a layer-wise variance decomposition to diagnose how predictive uncertainty accumulates across transformer depth.
Across SQuAD 2.0 answerability, MNLI, and SST-2, \ours achieves an average relative ECE reduction of approximately 20\% compared with a fine-tuned BERT-base baseline while preserving accuracy, and yields more informative uncertainty behavior for selective prediction under distribution shift.

\footnote{\href{https://github.com/eliashossain001/uq_decomposition/tree/main}{\texttt{github.com/eliashossain001/uq\_decomposition}}}
\end{abstract}

\section{Introduction}
\begin{figure}[t]
    \centering
    \includegraphics[width=\linewidth]{figures/uat_lite_overview.pdf}
    \caption{
    \textbf{Left, A:} Standard inference: deterministic forward pass without uncertainty.
    \textbf{Center, B:} Output-level uncertainty estimation via stochastic inference or
    post-hoc calibration.
    \textbf{Right, C:} \ours: uncertainty-aware inference in which
    epistemic uncertainty modulates attention with diagnostic insight.
    }
    \label{fig:introduction_figure}
    \vspace{-12pt}
\end{figure}

\label{sec:introduction}

Pretrained transformer-based language models achieve state-of-the-art performance
across a wide range of NLP tasks, from question answering to clinical decision support
\cite{devlin2019bert,brown2020language}. Despite these advances, modern transformers
exhibit \textbf{systematic miscalibration}, often assigning high confidence to
incorrect predictions \cite{desai2020calibration}. This mismatch between confidence
and correctness is particularly problematic for selective prediction and other
high-stakes settings.

\cref{fig:introduction_figure} summarizes common approaches to uncertainty estimation.
Standard transformer inference is deterministic and does not represent uncertainty in
internal computation (\cref{fig:introduction_figure}A). Post-hoc calibration methods,
such as temperature scaling (\TS), adjust confidence only after representations are formed,
leaving internal computation unchanged (\cref{fig:introduction_figure}B; \citealp{guo2017calibration}).
Stochastic approaches, including Monte Carlo (MC) dropout and deep ensembles
\citep{gal2016dropout,lakshminarayanan2017simple}, improve uncertainty estimates but
typically treat uncertainty as an output-level signal without influencing attention.

This gap reflects a broader issue: transformers are trained to optimize predictive
accuracy rather than confidence reliability \cite{guo2017calibration}. While post-hoc
methods are lightweight and effective in-domain, they do not alter token interactions.
Bayesian approaches integrate uncertainty more deeply but often require architectural
changes or specialized training, which limits compatibility with pretrained models.
This motivates a central question: can epistemic uncertainty shape a transformer's
attention at inference time without retraining or modifying pretrained weights?

We argue that effective uncertainty quantification should allow uncertainty to
\textbf{shape attention}, not merely annotate final predictions. For ambiguous inputs,
a reliable model should modulate evidence aggregation rather than propagate early
attention decisions deterministically. Existing stochastic inference methods expose
uncertainty mainly at the output level and do not use it to modulate internal attention
patterns (\cref{fig:introduction_figure}C). Although temperature scaling is highly effective
for output calibration, it only rescales logits and cannot provide token-level
epistemic signals or uncertainty-aware routing inside the model.

Motivated by these observations, we propose \ours, an inference-time framework for
encoder-based pretrained transformer classifiers. \ours estimates token-level
epistemic uncertainty via MC dropout and uses it to downweight attention contributions
from unstable tokens during contextualization, without modifying pretrained weights
or training objectives. Our main contributions are as follows:

\begin{itemize}[leftmargin=\parindent, itemsep=0.1em, topsep=0.4em]
    \item \textbf{Uncertainty-Weighted Attention.}
    An inference-time mechanism that injects token-level epistemic uncertainty, estimated via MC dropout,
    into self-attention, downweighting unstable token contributions during contextualization.

    \item \textbf{Layer-Wise Uncertainty Attribution.}
    A variance decomposition that attributes predictive uncertainty across transformer depth, enabling
    diagnostic analysis of where uncertainty amplifies during inference.

    \item \textbf{Comprehensive Evaluation.}
    Experiments on SQuAD~2.0 answerability, MNLI, and SST-2 that show improved calibration and selective prediction,
    along with robustness analyses under distribution shift.
\end{itemize}

\ours introduces no additional trainable parameters and operates strictly at inference time, using a \emph{controllable}
Monte Carlo budget (e.g., $M\in\{3,5,10\}$) to trade compute for reliability. \ours is not intended to replace
lightweight output-level calibrators such as \TS, which is optimized on development logits and often performs best in-domain.
While \TS rescales final logits after representations are formed, \ours modulates \emph{internal evidence aggregation}
by making attention uncertainty-aware. Concretely, \ours provides (i) token-level epistemic signals,
(ii) uncertainty-aware routing via attention modulation, and (iii) layer-wise diagnostic attribution; \TS provides
only output probability rescaling. Because the two mechanisms act at different stages, they are complementary and can
be stacked (\ours+\TS) when calibrated probabilities are needed. Overall, the framework trades a controllable
inference-time cost for improved reliability while avoiding the training and storage overhead of ensemble-based alternatives.
Our study focuses on encoder-based transformers; extending uncertainty-aware attention to decoder-only generative models
is left for future work.

\section{Related Work}

Uncertainty estimation and calibration for neural networks have been studied extensively.
Post-hoc methods such as temperature scaling and Platt scaling
\citep{guo2017calibration,platt1999probabilistic} provide simple and effective
confidence correction without altering model internals.
Ensemble-based approaches, including Deep Ensembles
\citep{lakshminarayanan2017simple}, often yield strong calibration performance but incur substantial
training and storage overhead.
Bayesian neural networks, stochastic inference methods, and recent Bayesian
transformer variants
\citep{blundell2015weight,gal2016dropout,ritter2018scalable,fan2020bayesian,zhang2021bayesian,chen2023calibrating}
offer principled approaches to uncertainty modeling, but they typically require architectural modifications,
variational training, or task-specific retraining.
By contrast, \ours integrates epistemic uncertainty directly into
self-attention at inference time, preserving pretrained weights and standard training
pipelines while enabling uncertainty-aware internal computation.
A more comprehensive discussion of related work and detailed comparisons is provided in \cref{app:related}.

\section{Preliminaries}

We briefly review the uncertainty concepts and stochastic inference protocol used throughout the paper. Bayesian perspectives distinguish \emph{epistemic uncertainty}, which arises from limited data or distribution shift, from \emph{aleatoric uncertainty}, which reflects irreducible observation noise. In our setting, epistemic effects dominate under semantic ambiguity and distribution shift, where transformers can be confidently wrong. We use MC dropout as a lightweight approximate Bayesian inference method. Specifically, dropout is retained at test time, and we run $M$ stochastic forward passes to obtain a predictive distribution whose mean is used for prediction and whose variability serves as an estimate of epistemic uncertainty. Standard MC dropout treats uncertainty as an output-level signal and does not influence internal computation. Because transformer inference is organized around token representations and attention interactions, we use token-level uncertainty derived from stochastic embeddings as a proxy for unstable token contributions. This motivates modulating attention with token uncertainty at inference time, without retraining or changing pretrained weights.

\section{\ours}
\label{sec:method}

\ours is an inference-time extension of pretrained transformer classifiers that
injects epistemic uncertainty into self-attention to improve confidence reliability
without modifying the model architecture, training objectives, or learned parameters.
The framework uses MC dropout to estimate token-level variability and to modulate
attention during contextualization. \ours also provides a diagnostic layer-wise
variance decomposition that characterizes how predictive uncertainty accumulates
across transformer depth without affecting the forward pass. Final predictions are
produced by aggregating stochastic logits, and optional \TS can be applied to the
MC-mean logits when calibrated probabilities are required.

\begin{figure}[t]
    \centering
    \includegraphics[width=1\columnwidth]{figures/uat_lite_architecture.pdf}
    \caption{
    Overview of \ours.
    Token-level epistemic uncertainty is estimated via MC dropout at the embedding layer
    and used to modulate self-attention in a pretrained transformer encoder.
    A layer-wise variance decomposition provides diagnostic attribution of predictive
    uncertainty across depth.
    }
    \vspace{-10pt}
    \label{fig:medbayes-v2}
\end{figure}


\subsection{Problem Setup and Notation}
\label{subsec:problem-setup}

We consider an encoder-based pretrained transformer for classification.
Given an input sequence $\mathbf{x}=(x_1,\dots,x_T)$, the model maps $\mathbf{x}$
through $L$ layers to contextual representations $\mathbf{h}^{(L)}$, which are
converted to a predictive distribution $\hat{y}$ by a task-specific head.
In standard evaluation, inference is deterministic: dropout is disabled, and the
same input yields the same prediction for fixed parameters.
While effective for accuracy, this setting provides no mechanism for representing
epistemic uncertainty during internal computation.
Our goal is to introduce controlled inference-time stochasticity and use it to
modulate attention without retraining or parameter updates.


\subsection{Estimating Uncertainty via MC Sampling}
\label{subsec:estimating-uncertainty}

We estimate predictive uncertainty using MC dropout, a lightweight approximation to
Bayesian inference.
Dropout is retained at test time, and we perform $M$ stochastic forward passes, each
induced by a different dropout mask.
Aggregating the resulting logits yields a predictive mean and an estimate of
epistemic variability across passes.

\paragraph{Uncertainty type.}
Following \citet{kendall2017uncertainties}, predictive uncertainty decomposes into
aleatoric and epistemic components.
MC dropout primarily captures epistemic uncertainty through approximate posterior
sampling at inference time.

\paragraph{Component-specific dropout.}
To balance variability and stability, we use component-specific dropout rates during
inference: 0.1 for embeddings, 0.2 for attention layers, and 0.3 for feed-forward layers.
These rates are fixed across tasks, and calibration trends remain stable under moderate
rate variations.

\paragraph{Single-pass sampling protocol.}

We use a single-pass MC procedure in which the same $M$ stochastic forward passes are
used to (i) estimate token-level embedding uncertainty and (ii) compute the final MC
predictive distribution.
Concretely, during pass $m$ we sample dropout masks and compute embeddings $\mathbf{z}^{(m)}$,
then update running uncertainty estimates $\hat{U}^{(m)}(x_j)$ from the embedding samples
observed so far.
To avoid circular dependence within a pass, we modulate attention in pass $m$ using the
\emph{lagged} estimate $\hat{U}^{(m-1)}(x_j)$, with $\hat{U}^{(0)}(x_j)=0$, as summarized in
Algorithm~\ref{alg:uat-lite}.
After $M$ passes, the final token uncertainty is $U(x_j)=\hat{U}^{(M)}(x_j)$, and predictions
are obtained by aggregating logits $\{\ell^{(m)}\}_{m=1}^{M}$.
This design avoids a separate embedding-only stage and ensures that overhead scales linearly in $M$.

\subsection{Uncertainty-Weighted Self-Attention}
\label{subsec:uncertainty_attention}

We incorporate epistemic uncertainty into self-attention by modulating attention
logits using token-level uncertainty estimated at the embedding layer.

\paragraph{Token-level uncertainty estimation.}
For each token $x_j$, we obtain $M$ stochastic embedding samples
$\{\mathbf{z}^{(m)}_j\}_{m=1}^{M}$ via MC dropout.
We define token uncertainty as
\begin{equation}
U(x_j)
=
\frac{1}{d}
\sum_{k=1}^{d}
\operatorname{Std}_{m=1}^{M}\!\left(z^{(m)}_{j,k}\right),
\end{equation}
where $d$ is the embedding dimension.
We treat $U(x_j)$ as an embedding-level proxy for the epistemic uncertainty induced by dropout
and use it as a \emph{control signal} for uncertainty-aware attention routing.

\textbf{Single-pass / online estimation.}
To avoid a separate embedding-only sampling stage, we estimate $U(x_j)$ online during the same $M$ stochastic passes used for prediction.
Specifically, after completing pass $m{-}1$, we maintain a running estimate $\hat{U}^{(m-1)}(x_j)$ (e.g., via Welford updates),
and we use $\hat{U}^{(m-1)}(x_j)$ to modulate attention in pass $m$.
After $M$ passes, the final estimate $U(x_j)=\hat{U}^{(M)}(x_j)$ aggregates variability across all samples.
During the MC run, each pass uses the lagged estimate $\hat{U}^{(m-1)}(x_j)$ consistently across all layers and heads.

\noindent\textbf{Scale/geometry diagnostics.}
We test whether $U(x_j)$ is a trivial artifact of embedding scale or geometry through
embedding-rescaling invariance, correlations with output-space epistemic measures
(entropy/MI), and uncertainty-definition ablations (\cref{app:scale-geometry}).

\paragraph{Attention modulation.}
Let $Q$, $K$, and $V$ denote the query, key, and value projections.
Standard scaled dot-product attention logits are
\begin{equation}
a_{ij}=\frac{Q_i K_j^\top}{\sqrt{d_k}}.
\end{equation}
We define uncertainty-aware logits by attenuating the pre-softmax scores:
\begin{equation}
\tilde{a}_{ij}=a_{ij}\exp(-\lambda u_{ij}),
\label{eq:uat_general}
\end{equation}
where $\lambda>0$ is a fixed penalty parameter and $u_{ij}$ specifies the projection variant.

\noindent\textbf{Positioning vs.\ \TS.}
\ours does not replace \TS: \TS rescales final logits post hoc, whereas \ours modulates
pre-softmax attention logits using token-level epistemic uncertainty, enabling
uncertainty-aware routing and token-level diagnostics beyond output rescaling.
We analyze sensitivity to $(\lambda,M)$ and compare penalty forms (exponential vs.\ linear
vs.\ reciprocal) in \cref{app:penalty_form}.

\paragraph{Projection variants (Q/K/V).}
Using $U(\cdot)$, we instantiate $u_{ij}$ as follows:
\begin{equation}
u_{ij}=
\begin{cases}
U(x_i) & \text{(Q-only)}\\
U(x_j) & \text{(K-only; evidence gating)}\\
U(x_i)+U(x_j) & \text{(QK)}
\end{cases}
\label{eq:uat_variants}
\end{equation}
We then obtain attention probabilities via the softmax:
\begin{equation}
\alpha_{ij}=\frac{\exp(\tilde{a}_{ij})}{\sum_{k}\exp(\tilde{a}_{ik})}.
\end{equation}
We additionally consider \textbf{V-only} and \textbf{QKV} by applying the same factor
to the value path, $\tilde{V}_j=V_j\exp(-\lambda U(x_j))$, and using
$\sum_j \alpha_{ij}\tilde{V}_j$ as the attention output.

\paragraph{Relation to prior work.}
Prior uncertainty-aware attention methods typically require architectural modification or
variational training \citep{chen2023calibrating,zhang2021bayesian}, whereas our approach operates
entirely at inference time. Unless otherwise stated, we default to Q-only, which empirically provides a favorable calibration--accuracy trade-off in our experiments in Section~\ref{sec:ablation-main}.


\subsection{Layer-Wise Uncertainty Attribution}
\label{subsec:layer-wise-attention}

We analyze how predictive uncertainty propagates across transformer depth under
inference-time stochasticity.
This analysis is purely diagnostic and does not alter model behavior.

\paragraph{Layer-wise variance decomposition.}
Using the law of total variance, we decompose the predictive variance under MC dropout
into layer-indexed contributions (\cref{app:mathematical-details}).


\begin{proposition}[Diagnostic Layer-Wise Variance Decomposition]
\label{thm:layer_variance}
Under the stochastic inference procedure used in this work, predictive variance
admits a recursive decomposition across transformer depth via repeated application
of the law of total variance.
\end{proposition}

\paragraph{Scope and interpretation.}
This decomposition follows from standard variance identities and is not claimed to be a
novel theoretical result.
Its purpose is to diagnose where uncertainty concentrates or amplifies across depth.
The resulting contributions should therefore be interpreted as approximate diagnostic signals
rather than causal attributions.


\subsection{Confidence-Aware Decision Shaping}
\label{subsec:confidence-aware-decision}

MC dropout yields logits $\{\ell^{(m)}\}_{m=1}^{M}$ from $M$ stochastic passes.
We compute the MC-mean logits as $\bar{\ell}=\frac{1}{M}\sum_{m=1}^{M}\ell^{(m)}$ and derive
the predictive distribution from $\bar{\ell}$.
When used, post-hoc temperature scaling is fit on the validation set and applied to
$\bar{\ell}$ at test time.

\paragraph{Selective prediction.}
We perform abstention by thresholding confidence on the validation split and applying
the same thresholds at test time for all methods.

\section{Experimental Setup}
\label{sec:experiments}

\subsection{Datasets}

We evaluate on five benchmarks spanning general-domain NLP and clinical-language settings
that are commonly used in prior calibration and robustness studies.
All experiments focus exclusively on encoder-based pretrained transformer
classifiers (BERT-family backbones), rather than decoder-only language models.
The general-domain datasets include SQuAD~2.0 \cite{rajpurkar2018know}, a question-answering
benchmark with unanswerable questions; MNLI \cite{williams2018broad}, a natural language inference
benchmark with matched and mismatched splits that support distribution-shift analysis; and
SST-2 \cite{socher-etal-2013-recursive}, a binary sentiment classification benchmark.
The clinical-domain benchmarks include MedQA \cite{jin2021disease} and PubMedQA
\cite{jin2019pubmedqa}, which we use to assess domain transfer rather than clinical validity.

\subsection{Baselines, Metrics, and Training Overview}
\label{sec:baseline-metrics-overview}

We compare our approach against BERT-base, MC Dropout, \TS,
Deep Ensembles, and representative Bayesian last-layer methods.
All models share the same pretrained backbone, fine-tuning recipe, and evaluation
protocol; detailed baseline configurations are provided in \cref{app:baseline-metrics}.
Calibration is evaluated using Expected Calibration Error (ECE), together with robustness
under distribution shift and selective prediction metrics.
Results are reported as mean $\pm$ standard deviation over five random seeds, while
deterministic baselines exhibit negligible variance.
The proposed method introduces no additional trainable parameters and operates
entirely at inference time.

\paragraph{Negation diagnostic.}
We include a controlled negation perturbation as a diagnostic of uncertainty responsiveness under semantic destabilization.
Because negation may alter the true label, we emphasize label-free stability signals
(e.g., $\Delta$Margin and $\Delta$Entropy) rather than treating this setting as a robustness
test of correctness.
Full details and examples are provided in \cref{app:negation}.

\paragraph{Uncertainty penalty.}
Unless otherwise stated, we use a single uncertainty penalty $\lambda$,
which is selected once on development data and then held fixed across all tasks and datasets.
\cref{app:sensitivity} shows that calibration performance remains stable across
a broad range of $\lambda$ values and Monte Carlo budgets $M$, indicating that the method does not rely on sensitive hyperparameter tuning in practice
for the evaluated settings.

In particular, across $\lambda\in\{0.1,0.5,1.0\}$ and $M\in\{3,5,10\}$,
ECE varies within a narrow band (range $<0.007$), while accuracy changes are negligible.
Increasing $M$ improves calibration, with diminishing returns beyond $M{=}5$,
suggesting that stable behavior does not require expensive Monte Carlo budgets or fine-grained tuning.


\section{Results}

\subsection{General NLP Calibration Performance}
\label{subsec:general_nlp_calibration}

Table~\ref{tab:main_results_with_clinical} reports multi-seed calibration results on SQuAD, MNLI, and SST-2
(ECE $\downarrow$; mean $\pm$ standard deviation over five seeds) under a shared evaluation protocol.
Across tasks, \ours improves calibration relative to the unscaled fine-tuned BERT baseline
(e.g., average ECE decreases from $0.1072$ to $0.0964$), with the largest improvement on MNLI
($0.0816 \rightarrow 0.0638$).
It also outperforms MC Dropout baselines, including uniform and component-specific variants,
suggesting that the gains are not explained by stochasticity alone.

Fig.~\ref{fig:mc_diagnostics_heatmap} shows that global- or document-level MC Dropout improves over the deterministic baseline but remains behind \TS, which achieves the lowest in-domain ECE.
We therefore view \ours as complementary to \TS: \ours modulates \emph{internal} evidence aggregation, whereas \TS rescales output logits.
The bottom panel of Table~\ref{tab:main_results_with_clinical} reports MedQA and PubMedQA
(single seed; accuracy reported in separate columns) as a clinical domain-transfer stress test.

MedQA analysis by linguistic bucket
(negation, numeric expressions, biomedical terminology, and hedging markers; Table~\ref{tab:behavior_medqa})
shows a consistent pattern:
\ours improves accuracy across all phenomena, while \ours+TS attains the lowest calibration error in every bucket.
Overall, \TS is strongest for in-domain probability calibration, \ours provides complementary uncertainty-aware inference, and \ours+TS yields the best trade-off in our evaluations.


\begin{table*}[t]
\centering
\small
\setlength{\tabcolsep}{6pt}
\renewcommand{\arraystretch}{1.14}

\begin{adjustbox}{max width=\linewidth, center}
\begin{tabular}{
p{0.34\linewidth}
>{\centering\arraybackslash}p{0.14\linewidth}
>{\centering\arraybackslash}p{0.14\linewidth}
>{\centering\arraybackslash}p{0.14\linewidth}
>{\centering\arraybackslash}p{0.10\linewidth}
}
\toprule
\rowcolor{tblsection}
\multicolumn{5}{l}{\textbf{Panel A: General NLP (multi-seed)}} \\
\addlinespace[0.35ex]
\textbf{Method}
& \textbf{SQuAD (ECE $\downarrow$)}
& \textbf{MNLI (ECE $\downarrow$)}
& \textbf{SST-2 (ECE $\downarrow$)}
& \textbf{Avg (ECE $\downarrow$)} \\
\midrule

BERT-base
& 0.1868 $\pm$ 0.000
& 0.0816 $\pm$ 0.000
& 0.0531 $\pm$ 0.000
& 0.1072 \\

\rowcolor{tblts}
Base+TS
& \textbf{0.0788} $\pm$ 0.000
& \textbf{0.0148} $\pm$ 0.000
& \textbf{0.0163} $\pm$ 0.000
& \textbf{0.0366} \\

MC Dropout (uniform $p{=}0.1$)
& 0.173 $\pm$ 0.001
& 0.090 $\pm$ 0.001
& 0.039 $\pm$ 0.001
& 0.101 \\

MC Dropout (component-specific)
& 0.186 $\pm$ 0.000
& 0.109 $\pm$ 0.001
& 0.054 $\pm$ 0.001
& 0.116 \\

SNGP
& \textbf{0.001} $\pm$ 0.000
& 0.021 $\pm$ 0.000
& 0.009 $\pm$ 0.000
& 0.010 \\

VI-LastLayer
& 0.119 $\pm$ 0.001
& 0.096 $\pm$ 0.000
& 0.136 $\pm$ 0.002
& 0.117 \\

Deep Ensemble (5)
& 0.168 $\pm$ 0.000
& 0.060 $\pm$ 0.000
& 0.037 $\pm$ 0.000
& 0.089 \\

Global MC Dropout (M=10)
& 0.1600 $\pm$ 0.000
& 0.0565 $\pm$ 0.000
& 0.0501 $\pm$ 0.000
& 0.0889 \\

\rowcolor{tblours}
\textbf{\ours} (M=10)
& 0.1647 $\pm$ 0.000
& 0.0638 $\pm$ 0.000
& 0.0608 $\pm$ 0.000
& 0.0964 \\

\rowcolor{tbloursTS}
\textbf{\ours}+TS (M=10)
& \textbf{0.0760} $\pm$ 0.000
& \textbf{0.0201} $\pm$ 0.000
& \textbf{0.0191} $\pm$ 0.000
& \textbf{0.0384} \\

\addlinespace[0.6ex]
\midrule
\addlinespace[0.6ex]

\rowcolor{tblsection}
\multicolumn{5}{l}{\textbf{Panel B: Clinical QA (single seed)}} \\
\addlinespace[0.35ex]

\textbf{Method}
& \textbf{MedQA (ECE $\downarrow$)}
& \textbf{MedQA (Acc $\uparrow$)}
& \textbf{PubMedQA (ECE $\downarrow$)}
& \textbf{PubMedQA (Acc $\uparrow$)} \\
\midrule

Baseline
& 0.033 & 0.246
& 0.114 & 0.620 \\

\rowcolor{tblts}
Base+TS
& 0.023 & 0.246
& 0.128 & 0.620 \\

Deep Ensemble (5)
& \textbf{0.007} & \textbf{0.268}
& \textbf{0.051} & 0.520 \\

\rowcolor{tblours}
\textbf{\ours}
& 0.040 & 0.249
& 0.103 & \textbf{0.640} \\

\rowcolor{tbloursTS}
\textbf{\ours}+TS
& 0.0069 & 0.251
& 0.1018 & 0.630 \\

\bottomrule
\end{tabular}
\end{adjustbox}

\caption{
\textbf{Main calibration results (general NLP) and clinical QA results.}
\emph{Panel A (multi-seed):} SQuAD, MNLI, SST-2 (ECE $\downarrow$; mean $\pm$ std over five seeds; $K{=}15$ fixed-width confidence bins).
Temperature scaling is fit on validation data and applied at test time.
For \ours+TS, the temperature $T$ is fit on dev-set MC-mean logits and applied at test time.
\emph{Panel B (single seed):} MedQA and PubMedQA, reported with ECE $\downarrow$ and accuracy $\uparrow$ in separate columns, are used as a clinical domain-transfer stress test.
\emph{Note:} SNGP calibration can be sensitive to implementation details (e.g., confidence definition and binning);
we report reproduced numbers under the same evaluation pipeline.
}
\label{tab:main_results_with_clinical}
\vspace{-2mm}
\end{table*}

\subsection{Distribution Shift Robustness}
\label{subsec:shift_robustness}

We evaluate uncertainty under covariate shift in two settings:
(i) \textbf{MNLI transfer} (matched$\rightarrow$mismatched; mild shift) and
(ii) four \textbf{OOD suites} with stronger shifts (\textbf{HANS}, \textbf{ANLI}, \textbf{SNLI}, \textbf{IMDb}).
Table~\ref{tab:dist_shift} reports MNLI transfer results (mean $\pm$ standard deviation over five seeds), and
Table~\ref{tab:5n2p-ood-thresh} reports selective prediction on OOD suites, including admitted counts $N@\tau$.

\paragraph{Calibration baselines and composability with TS.}
\textbf{Temperature scaling (TS)} is the strongest post-hoc baseline for reducing marginal ECE via global logit rescaling.
By contrast, \ours modulates \emph{attention-level evidence aggregation} using MC-dropout uncertainty and composes naturally with TS.
On MNLI transfer, \textbf{\ours+TS matches TS-level average ECE} with essentially unchanged accuracy
(Table~\ref{tab:dist_shift}).

\paragraph{Selective prediction at fixed confidence thresholds.}
We report \textbf{Cov@$\tau$}, \textbf{Acc@$\tau$}, \textbf{$N@\tau$}, and \textbf{AURC}
for $\tau\in\{0.9,0.8,0.7\}$ (Table~\ref{tab:5n2p-ood-thresh}).
Across OOD suites, \textbf{TS reliably reduces ECE} but does not uniformly dominate
selective prediction, especially when score compression reduces high-$\tau$ coverage.
Empirically, \textbf{\ours+TS} is the most reliable configuration, preserving TS-level calibration while retaining attention-level gating.

\paragraph{Shortcut-stress test on HANS.}
HANS probes heuristic shortcut reliance.
We additionally report a per-heuristic breakdown and a \textbf{non-entailment-only} subset (Table~\ref{tab:5n2p-hans-heur}),
which isolates the dominant failure mode of overconfident non-entailment errors.
On this diagnostic, \textbf{\ours and \ours+TS improve accuracy over Base and Base+TS},
suggesting that uncertainty-weighted attention mitigates shortcut-driven evidence beyond output-only rescaling.

\begin{tcolorbox}[
    colback=gray!8,
    colframe=black!70,
    boxrule=0.5pt,
    arc=2pt,
    left=6pt,
    right=6pt,
    top=4pt,
    bottom=4pt,
    title=Core Finding,
    fonttitle=\bfseries
]
Across distribution shifts, TS is strongest for marginal ECE, whereas \ours provides complementary uncertainty-aware inference. Overall, \textbf{\ours+TS} is the most reliable configuration in our evaluations.
\end{tcolorbox}

\begin{table*}[t]
\centering
\setlength{\tabcolsep}{5.5pt}
\renewcommand{\arraystretch}{1.10}
\small

\resizebox{\textwidth}{!}{%
\begin{tabular}{lcccccc}
\toprule
\rowcolor{tblHeader}
\textbf{Method} &
\textbf{ID ECE$\downarrow$} &
\textbf{OOD ECE$\downarrow$} &
$\boldsymbol{\Delta}$\textbf{ECE$\downarrow$} &
\textbf{Avg ECE$\downarrow$} &
\textbf{ID Acc$\uparrow$} &
\textbf{OOD Acc$\uparrow$} \\
\midrule

\rowcolor{tblBlue}
Base
& 0.0809$\pm$0.0040 & 0.0555$\pm$0.0000 & $-0.0254$$\pm$0.0040 & 0.0682$\pm$0.0020 & 0.7162$\pm$0.0048 & 0.7459$\pm$0.0000 \\

\rowcolor{tblBlue}
Base+TS
& \textbf{0.0147}$\pm$0.0031 & 0.0208$\pm$0.0022 & 0.0061$\pm$0.0040 & \textbf{0.0177}$\pm$0.0018 & 0.7162$\pm$0.0048 & 0.7459$\pm$0.0000 \\

\rowcolor{tblBlue}
Global MC Dropout (M=10)
& 0.0586$\pm$0.0037 & \textbf{0.0334}$\pm$0.0023 & $-0.0253$$\pm$0.0033 & 0.0460$\pm$0.0026 & 0.7090$\pm$0.0047 & 0.7398$\pm$0.0012 \\

\rowcolor{tblRed}
UAT-LITE (M=10)
& 0.0650$\pm$0.0028 & 0.0402$\pm$0.0012 & $-0.0247$$\pm$0.0038 & 0.0526$\pm$0.0010 & 0.7099$\pm$0.0038 & 0.7395$\pm$0.0014 \\

\rowcolor{tblRed}
UAT-LITE+TS (M=10)
& 0.0163$\pm$0.0022 & 0.0199$\pm$0.0013 & 0.0035$\pm$0.0017 & 0.0181$\pm$0.0016 & 0.7105$\pm$0.0047 & 0.7394$\pm$0.0004 \\

\bottomrule
\end{tabular}
}

\caption{Multi-seed distribution-shift robustness on MNLI (matched$\rightarrow$mismatched).
Results are reported as mean $\pm$ standard deviation over five seeds.
$\Delta$ECE $=$ ECE$_{\mathrm{OOD}} - \mathrm{ECE}_{\mathrm{ID}}$ (lower is better), and Avg ECE is defined as $\tfrac{1}{2}(\mathrm{ECE}_{\mathrm{ID}}+\mathrm{ECE}_{\mathrm{OOD}})$.
For UAT-LITE+TS, the temperature $T$ is fit on the development set using MC-mean logits and applied at test time.}
\label{tab:dist_shift}
\vspace{-2mm}
\end{table*}

\subsection{Selective Prediction Performance}
\label{subsec:selective-prediction}

Table~\ref{tab:5n2p-thresh} reports threshold-based selective prediction results on MNLI and SQuAD~2.0,
including \TS and global MC Dropout baselines.
Confidence is defined as $s(x)=\max_y p_\theta(y\mid x)$.
For $\tau\in\{0.9,0.8,0.7\}$, we report coverage (Cov@$\tau$), admitted accuracy (Acc@$\tau$),
admitted count ($N@\tau$), and AURC (risk--coverage; $\downarrow$).
Thresholds are selected on validation data \emph{for each method} and then applied at test time; we also report $N@\tau$ to make coverage shifts explicit.

Selective prediction evaluates whether uncertainty aligns with prediction risk:
a reliable model should abstain on higher-risk inputs while retaining coverage on lower-risk, high-confidence examples.
This setting is particularly informative under fixed confidence thresholds, where improved marginal calibration does not necessarily imply better abstention behavior.

\paragraph{Effect of TS under fixed thresholds.}
We include \TS because, although it preserves the predicted label, probability rescaling changes which examples exceed a fixed threshold $\tau$.
As a result, \TS can substantially alter threshold-level behavior even when accuracy remains unchanged.
On SQuAD~2.0, for example, \TS can markedly reduce Cov@0.9, reflecting confidence compression rather than a reporting artifact.

\paragraph{Effect of uncertainty-aware inference.}
By contrast, \ours modifies \emph{internal} evidence aggregation through uncertainty-aware attention.
This changes how confidence is formed before thresholding and can improve alignment between confidence and risk at fixed operating points.
Overall, \ours, and especially \ours+\TS, yield more reliable risk--coverage behavior, whereas \TS alone can substantially shift coverage through score compression; see Table~\ref{tab:5n2p-thresh} and Table~\ref{tab:5n2p-ood-thresh}.

%
%

\subsection{Computational Efficiency and Latency}
\label{subsec:latency}

Because \ours relies on Monte Carlo sampling, it incurs additional inference-time cost that scales
approximately linearly with the MC budget $M$ under a sequential sampling loop (i.e., $M$ stochastic forward passes).
To quantify end-to-end overhead in a single canonical setting, we report wall-clock latency on an NVIDIA A100 GPU
with batch size $32$ and sequence length $128$.
\ours uses $M{=}10$ stochastic passes implemented as a sequential Monte Carlo loop; under this setting, the deterministic
model requires $62.9$ ms per batch, whereas \ours requires $1426.9$ ms, corresponding to a $22.7\times$ overhead.
Temperature scaling adds negligible test-time cost ($<1$ ms) because it only applies a scalar rescaling of the logits. Table~\ref{tab:efficiency} reports the resulting overhead.
At $M{=}10$, \ours incurs a substantial slowdown relative to deterministic inference, whereas \TS remains essentially free.
Accordingly, \ours is best suited to offline or batch evaluation, or to selectively triggered regimes (e.g., invoked only for
low-confidence or high-risk inputs), rather than as an always-on mechanism in latency-critical deployments.

\subsection{Ablation Study}
\label{sec:ablation-main}

\cref{fig:ablation} decomposes \ours into its constituent components to isolate their
individual contributions.
Uncertainty-weighted attention alone reduces ECE under controlled comparisons, indicating
that attention modulation based on epistemic uncertainty is a major driver of the
observed calibration gains.
By contrast, embedding-level stochasticity alone can slightly degrade calibration, suggesting that
stochasticity without structured uncertainty-aware computation is insufficient.
Combining uncertainty-weighted attention with confidence-aware decision shaping yields the
largest improvement among the evaluated components, consistent with their complementarity.

To examine the design choice of \emph{where} to apply uncertainty modulation within attention,
we evaluate variants that gate the query (Q), key (K), value (V), or all three jointly (QKV).
\cref{fig:qkv-ablation} summarizes the calibration--accuracy trade-offs across SST-2, MNLI, and SQuAD~2.0.
Overall, K-only achieves the lowest ECE on SST-2 and MNLI but can substantially reduce accuracy on SQuAD~2.0,
suggesting overly aggressive evidence gating in span-style question answering.
Q-only provides the most balanced behavior, improving calibration while largely preserving accuracy across tasks,
whereas V-only best preserves, and in some cases slightly improves, accuracy while yielding more modest ECE gains.
Joint QKV gating can further reduce ECE on some datasets but tends to over-regularize and degrade accuracy.

\begin{figure*}[t]
  \centering
  \includegraphics[width=0.98\textwidth]{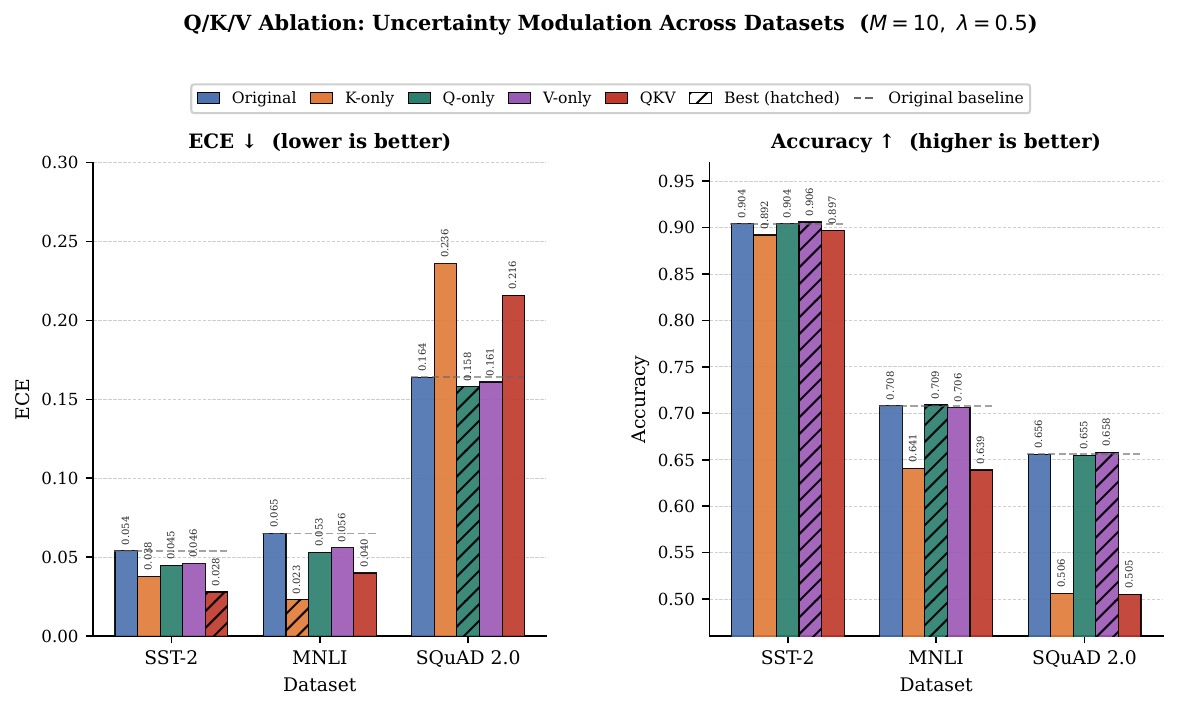}
    \caption{\textbf{Q/K/V ablation: where to apply uncertainty modulation} ($M{=}10$, $\lambda{=}0.5$).
    Top: ECE$\downarrow$; bottom: Accuracy$\uparrow$ on SST-2, MNLI, and SQuAD~2.0 for Original, Q-only, K-only, V-only, and QKV.
    K-only achieves the lowest ECE on SST-2 and MNLI but sharply reduces accuracy on SQuAD~2.0; Q-only provides the most consistent trade-off across tasks; V-only best preserves accuracy; and QKV can over-regularize.}
  \label{fig:qkv-ablation}
  \vspace{-2mm}
\end{figure*}






\subsection{Generalization Across BERT Models}

We evaluate generalization across \emph{BERT-based encoder architectures}
spanning general-purpose, biomedical, clinical, and scientific domains
(\cref{app:appendix_cross_model}).
Under matched evaluation protocols, \ours improves calibration across these
pretrained backbones, yielding an average relative ECE reduction of
approximately \textbf{32\%}.

The largest gains are observed for mid-sized and domain-specialized models, including
BERT-base on SQuAD, BioBERT on PubMedQA, and ClinicalBERT on MedQA, where epistemic
uncertainty associated with semantic ambiguity appears more prominent.
By contrast, improvements are smaller for models with stronger baseline calibration
(e.g., RoBERTa-base on MNLI) and diminish, or occasionally reverse, for very small
or heavily overparameterized models.

These results suggest that uncertainty-aware attention generalizes reliably
\emph{within the BERT encoder family}.
Extending this mechanism to fundamentally different architectures, such as
decoder-only or state-space models, remains an important direction for future work.

\paragraph{Additional analyses.}
Additional analyses are provided in \cref{app:appendix_results,app:appendix_cross_model,app:complexity-analysis}, including threshold-based selective prediction, OOD evaluations (e.g., HANS), global and document-level MC dropout diagnostics, component ablations, penalty-form ablations, aggregation and normalization diagnostics, scale and geometry diagnostics for token uncertainty, negation diagnostics with qualitative attention illustration, behavioral analysis by linguistic bucket on MedQA, sensitivity analyses for $\lambda$ and $M$, character-level adversarial robustness, full linguistic uncertainty analysis, cross-model and scaling analyses, and implementation, complexity, latency, and reproducibility details.

\section{Conclusion}
\label{sec:conclusion}

We introduced \ours, an inference-time framework that integrates epistemic uncertainty into transformer inference through uncertainty-weighted attention and a layer-wise variance diagnostic.
Across general NLP benchmarks, \ours improves calibration relative to an unscaled fine-tuned BERT baseline and yields more reliable selective prediction behavior under fixed confidence thresholds.
Under distribution shift (MNLI matched$\rightarrow$mismatched and OOD evaluation suites), uncertainty-aware attention provides benefits that are complementary to output-level calibration: temperature scaling remains the strongest post-hoc baseline for ECE reduction, whereas \ours targets internal evidence aggregation and composes cleanly with \TS when calibrated probabilities are required.
Although Monte Carlo sampling increases inference latency, \ours requires neither retraining nor modification of pretrained weights, providing a practical bridge between lightweight post-hoc calibration methods and higher-cost Bayesian or ensemble-based alternatives.

\section*{Limitations}
\label{sec:discussion}

\paragraph{Compute and latency.}
\ours introduces additional inference cost due to Monte Carlo dropout sampling.
Under a single canonical end-to-end GPU benchmark (batch size $32$, sequence length $128$, sequential MC loop on one GPU), the overhead is substantial (e.g., approximately $22.7\times$ at $M{=}10$), whereas temperature scaling adds negligible cost (approximately $1.0\times$).
As a result, \ours is not intended for latency-critical, always-on deployment.
Instead, it is better suited to offline or batch settings, or to selectively triggered use (e.g., enabling uncertainty-aware attention only for low-confidence or high-risk inputs), where improved confidence reliability and internal evidence control may justify the additional runtime.
Overall, \ours occupies an intermediate point between lightweight post-hoc calibration methods (e.g., \TS) and higher-cost uncertainty methods (e.g., deep ensembles), trading additional compute for internal uncertainty-aware control signals and diagnostic capability.

\paragraph{Compute-once token uncertainty (no per-layer recomputation).}
For efficiency and conceptual clarity, we compute token-level uncertainty once at the embedding stage (via $M$ stochastic embedding samples) and reuse this signal to modulate attention throughout the network; deeper layers are analyzed through variance attribution rather than by re-estimating $U(x_j)$ after each mixing operation.
Recomputing token uncertainty at intermediate layers is a potentially useful extension, but it would require defining token-level uncertainty after contextual mixing (e.g., residual connections and attention aggregation) and would further increase inference cost due to additional stochastic passes.
We leave per-layer recomputation, and its associated cost--benefit trade-offs, to future work, especially in settings where additional compute is acceptable.

\paragraph{Not a replacement for TS in-domain.}
Because temperature scaling is a strong lightweight calibrator for in-domain ECE, \ours should not be viewed as a universal substitute for \TS.
Our primary contribution is different: uncertainty-guided modulation of the \emph{internal evidence aggregation} process through attention at inference time.
Accordingly, we emphasize compositional use: \ours can be combined with \TS (by fitting $T$ on development data using MC-mean logits) when the goal is to retain internal uncertainty-aware behavior while further improving output calibration.

\paragraph{Clinical transfer is a stress test, not clinical validation.}
On clinical QA benchmarks (MedQA and PubMedQA), calibration gains can be modest, which may reflect dataset properties (e.g., a stronger emphasis on factual recall and less ambiguity in evidence aggregation) as well as evaluation constraints.
We therefore treat these experiments as domain-transfer stress tests rather than evidence of clinical readiness.
Any real-world clinical deployment would require task- and setting-specific validation, including broader coverage, robustness checks, and appropriate governance safeguards.

\bibliography{custom}

\begin{thebibliography}{33}
\providecommand{\natexlab}[1]{#1}

\bibitem[{Blundell et~al.(2015)Blundell, Cornebise, Kavukcuoglu, and Wierstra}]{blundell2015weight}
Charles Blundell, Julien Cornebise, Koray Kavukcuoglu, and Daan Wierstra. 2015.
\newblock Weight uncertainty in neural network.
\newblock In \emph{International conference on machine learning}, pages 1613--1622. PMLR.

\bibitem[{Brown et~al.(2020)Brown, Mann, Ryder, Subbiah, Kaplan, Dhariwal, Neelakantan, Shyam, Sastry, Askell et~al.}]{brown2020language}
Tom Brown, Benjamin Mann, Nick Ryder, Melanie Subbiah, Jared~D Kaplan, Prafulla Dhariwal, Arvind Neelakantan, Pranav Shyam, Girish Sastry, Amanda Askell, and 1 others. 2020.
\newblock Language models are few-shot learners.
\newblock \emph{Advances in neural information processing systems}, 33:1877--1901.

\bibitem[{Chen and Li(2023)}]{chen2023calibrating}
Wenlong Chen and Yingzhen Li. 2023.
\newblock Calibrating transformers via sparse gaussian processes.
\newblock \emph{arXiv preprint arXiv:2303.02444}.

\bibitem[{Desai and Durrett(2020)}]{desai2020calibration}
Shrey Desai and Greg Durrett. 2020.
\newblock Calibration of pre-trained transformers.
\newblock \emph{arXiv preprint arXiv:2003.07892}.

\bibitem[{Devlin et~al.(2019)Devlin, Chang, Lee, and Toutanova}]{devlin2019bert}
Jacob Devlin, Ming-Wei Chang, Kenton Lee, and Kristina Toutanova. 2019.
\newblock Bert: Pre-training of deep bidirectional transformers for language understanding.
\newblock In \emph{Proceedings of the 2019 conference of the North American chapter of the association for computational linguistics: human language technologies, volume 1 (long and short papers)}, pages 4171--4186.

\bibitem[{Fan et~al.(2020)Fan, Zhang, Chen, and Zhou}]{fan2020bayesian}
Xinjie Fan, Shujian Zhang, Bo~Chen, and Mingyuan Zhou. 2020.
\newblock Bayesian attention modules.
\newblock \emph{Advances in Neural Information Processing Systems}, 33:16362--16376.

\bibitem[{Gabetni et~al.(2025)Gabetni, Curci, Pilzer, Roy, Ricci, and Franchi}]{gabetni2025ensemblingprunedattentionheads}
Firas Gabetni, Giuseppe Curci, Andrea Pilzer, Subhankar Roy, Elisa Ricci, and Gianni Franchi. 2025.
\newblock \href {https://arxiv.org/abs/2510.18358} {Ensembling pruned attention heads for uncertainty-aware efficient transformers}.
\newblock \emph{Preprint}, arXiv:2510.18358.

\bibitem[{Gal and Ghahramani(2016)}]{gal2016dropout}
Yarin Gal and Zoubin Ghahramani. 2016.
\newblock Dropout as a bayesian approximation: Representing model uncertainty in deep learning.
\newblock In \emph{international conference on machine learning}, pages 1050--1059. PMLR.

\bibitem[{Goodfellow et~al.(2014)Goodfellow, Shlens, and Szegedy}]{goodfellow2014explaining}
Ian~J Goodfellow, Jonathon Shlens, and Christian Szegedy. 2014.
\newblock Explaining and harnessing adversarial examples.
\newblock \emph{arXiv preprint arXiv:1412.6572}.

\bibitem[{Guo et~al.(2017)Guo, Pleiss, Sun, and Weinberger}]{guo2017calibration}
Chuan Guo, Geoff Pleiss, Yu~Sun, and Kilian~Q Weinberger. 2017.
\newblock On calibration of modern neural networks.
\newblock In \emph{International conference on machine learning}, pages 1321--1330. PMLR.

\bibitem[{Huang et~al.(2017)Huang, Li, Pleiss, Liu, Hopcroft, and Weinberger}]{huang2017snapshot}
Gao Huang, Yixuan Li, Geoff Pleiss, Zhuang Liu, John~E Hopcroft, and Kilian~Q Weinberger. 2017.
\newblock Snapshot ensembles: Train 1, get m for free.
\newblock \emph{arXiv preprint arXiv:1704.00109}.

\bibitem[{Jin et~al.(2021)Jin, Pan, Oufattole, Weng, Fang, and Szolovits}]{jin2021disease}
Di~Jin, Eileen Pan, Nassim Oufattole, Wei-Hung Weng, Hanyi Fang, and Peter Szolovits. 2021.
\newblock What disease does this patient have? a large-scale open domain question answering dataset from medical exams.
\newblock \emph{Applied Sciences}.

\bibitem[{Jin et~al.(2019)Jin, Dhingra, Liu, Cohen, and Lu}]{jin2019pubmedqa}
Qiao Jin, Bhuwan Dhingra, Zhengping Liu, William Cohen, and Xinghua Lu. 2019.
\newblock Pubmedqa: A dataset for biomedical research question answering.
\newblock In \emph{Proceedings of the 2019 conference on empirical methods in natural language processing and the 9th international joint conference on natural language processing (EMNLP-IJCNLP)}, pages 2567--2577.

\bibitem[{Kendall and Gal(2017)}]{kendall2017uncertainties}
Alex Kendall and Yarin Gal. 2017.
\newblock What uncertainties do we need in bayesian deep learning for computer vision?
\newblock \emph{Advances in neural information processing systems}, 30.

\bibitem[{Ko et~al.(2024)Ko, Parkinson, Liu, and Wang}]{10.1093/bib/bbae359}
Young~Su Ko, Jonathan Parkinson, Cong Liu, and Wei Wang. 2024.
\newblock \href {https://doi.org/10.1093/bib/bbae359} {Tuna: an uncertainty-aware transformer model for sequence-based protein–protein interaction prediction}.
\newblock \emph{Briefings in Bioinformatics}, 25(5):bbae359.

\bibitem[{Kuhn et~al.(2023)Kuhn, Gal, and Farquhar}]{kuhn2023semantic}
Lorenz Kuhn, Yarin Gal, and Sebastian Farquhar. 2023.
\newblock Semantic uncertainty: Linguistic invariances for uncertainty estimation in natural language generation.
\newblock \emph{arXiv preprint arXiv:2302.09664}.

\bibitem[{Lakshminarayanan et~al.(2017)Lakshminarayanan, Pritzel, and Blundell}]{lakshminarayanan2017simple}
Balaji Lakshminarayanan, Alexander Pritzel, and Charles Blundell. 2017.
\newblock Simple and scalable predictive uncertainty estimation using deep ensembles.
\newblock \emph{Advances in neural information processing systems}, 30.

\bibitem[{Louizos and Welling(2017)}]{louizos2017multiplicative}
Christos Louizos and Max Welling. 2017.
\newblock Multiplicative normalizing flows for variational bayesian neural networks.
\newblock In \emph{International conference on machine learning}, pages 2218--2227. PMLR.

\bibitem[{MacKay(1992)}]{mackay1992practical}
David~JC MacKay. 1992.
\newblock A practical bayesian framework for backpropagation networks.
\newblock \emph{Neural computation}, 4(3):448--472.

\bibitem[{Neal(2012)}]{neal2012bayesian}
Radford~M Neal. 2012.
\newblock \emph{Bayesian learning for neural networks}, volume 118.
\newblock Springer Science \& Business Media.

\bibitem[{Platt et~al.(1999)}]{platt1999probabilistic}
John Platt and 1 others. 1999.
\newblock Probabilistic outputs for support vector machines and comparisons to regularized likelihood methods.
\newblock \emph{Advances in large margin classifiers}, 10(3):61--74.

\bibitem[{Rahmati et~al.(2025)Rahmati, Jantre, Zhang, Wang, Yoon, Urban, and Qian}]{rahmati2025c}
Amir~Hossein Rahmati, Sanket Jantre, Weifeng Zhang, Yucheng Wang, Byung-Jun Yoon, Nathan~M Urban, and Xiaoning Qian. 2025.
\newblock C-lora: Contextual low-rank adaptation for uncertainty estimation in large language models.
\newblock \emph{arXiv preprint arXiv:2505.17773}.

\bibitem[{Rajpurkar et~al.(2018)Rajpurkar, Jia, and Liang}]{rajpurkar2018know}
Pranav Rajpurkar, Robin Jia, and Percy Liang. 2018.
\newblock Know what you don't know: Unanswerable questions for squad.
\newblock \emph{arXiv preprint arXiv:1806.03822}.

\bibitem[{Ritter et~al.(2018)Ritter, Botev, and Barber}]{ritter2018scalable}
Hippolyt Ritter, Aleksandar Botev, and David Barber. 2018.
\newblock A scalable laplace approximation for neural networks.
\newblock In \emph{6th international conference on learning representations, ICLR 2018-conference track proceedings}, volume~6. International Conference on Representation Learning.

\bibitem[{Socher et~al.(2013)Socher, Perelygin, Wu, Chuang, Manning, Ng, and Potts}]{socher-etal-2013-recursive}
Richard Socher, Alex Perelygin, Jean Wu, Jason Chuang, Christopher~D. Manning, Andrew Ng, and Christopher Potts. 2013.
\newblock \href {https://aclanthology.org/D13-1170/} {Recursive deep models for semantic compositionality over a sentiment treebank}.
\newblock In \emph{Proceedings of the 2013 Conference on Empirical Methods in Natural Language Processing}, pages 1631--1642, Seattle, Washington, USA. Association for Computational Linguistics.

\bibitem[{Talman et~al.(2023)Talman, Celikkanat, Virpioja, Heinonen, and Tiedemann}]{talman2023uncertainty}
Aarne Talman, Hande Celikkanat, Sami Virpioja, Markus Heinonen, and J{\"o}rg Tiedemann. 2023.
\newblock Uncertainty-aware natural language inference with stochastic weight averaging.
\newblock \emph{arXiv preprint arXiv:2304.04726}.

\bibitem[{Vazhentsev et~al.(2025)Vazhentsev, Rvanova, Kuzmin, Fadeeva, Lazichny, Panchenko, Panov, Baldwin, Sachan, Nakov et~al.}]{vazhentsev2025uncertainty}
Artem Vazhentsev, Lyudmila Rvanova, Gleb Kuzmin, Ekaterina Fadeeva, Ivan Lazichny, Alexander Panchenko, Maxim Panov, Timothy Baldwin, Mrinmaya Sachan, Preslav Nakov, and 1 others. 2025.
\newblock Uncertainty-aware attention heads: Efficient unsupervised uncertainty quantification for llms.
\newblock \emph{arXiv preprint arXiv:2505.20045}.

\bibitem[{Wang et~al.(2024)Wang, Zheng, Ding, Zhang, Lin, and Tao}]{wang2024uncertainty}
Yikun Wang, Rui Zheng, Liang Ding, Qi~Zhang, Dahua Lin, and Dacheng Tao. 2024.
\newblock Uncertainty aware learning for language model alignment.
\newblock \emph{arXiv preprint arXiv:2406.04854}.

\bibitem[{Williams et~al.(2018)Williams, Nangia, and Bowman}]{williams2018broad}
Adina Williams, Nikita Nangia, and Samuel Bowman. 2018.
\newblock A broad-coverage challenge corpus for sentence understanding through inference.
\newblock In \emph{Proceedings of the 2018 conference of the North American chapter of the association for computational linguistics: human language technologies, volume 1 (long papers)}, pages 1112--1122.

\bibitem[{Xue et~al.(2021)Xue, Yu, Xu, Liu, Hu, Ye, Geng, Liu, and Meng}]{xue2021bayesian}
Boyang Xue, Jianwei Yu, Junhao Xu, Shansong Liu, Shoukang Hu, Zi~Ye, Mengzhe Geng, Xunying Liu, and Helen Meng. 2021.
\newblock Bayesian transformer language models for speech recognition.
\newblock In \emph{ICASSP 2021-2021 IEEE International Conference on Acoustics, Speech and Signal Processing (ICASSP)}, pages 7378--7382. IEEE.

\bibitem[{Zabolotnyi et~al.(2025)Zabolotnyi, Makarov, Mitrovic, Proskura, Travkin, Alferov, and Zaytsev}]{zabolotnyi2025adue}
Artem Zabolotnyi, Roman Makarov, Mile Mitrovic, Polina Proskura, Oleg Travkin, Roman Alferov, and Alexey Zaytsev. 2025.
\newblock Adue: Improving uncertainty estimation head for lora adapters in llms.
\newblock \emph{arXiv preprint arXiv:2505.15443}.

\bibitem[{Zadrozny and Elkan(2001)}]{zadrozny2001obtaining}
Bianca Zadrozny and Charles Elkan. 2001.
\newblock Obtaining calibrated probability estimates from decision trees and naive bayesian classifiers.
\newblock In \emph{Icml}, volume~1.

\bibitem[{Zhang et~al.(2021)Zhang, Fan, Chen, and Zhou}]{zhang2021bayesian}
Shujian Zhang, Xinjie Fan, Bo~Chen, and Mingyuan Zhou. 2021.
\newblock Bayesian attention belief networks.
\newblock In \emph{International Conference on Machine Learning}, pages 12413--12426. PMLR.

\end{thebibliography}

\newpage

\appendix
\onecolumn
\section*{Appendix}

This supplementary material provides additional experimental results, methodological details, and reproducibility information supporting the main paper.

\begin{itemize}

    \item \textbf{Appendix~\ref{app:mathematical-details}} presents the mathematical foundations of \ours,
    including the diagnostic layer-wise variance decomposition, its derivation, and the Bayesian interpretation of MC dropout.

    \item \textbf{Appendix~\ref{app:algorithm}} provides the inference-time algorithm for uncertainty-aware attention.

    \item \textbf{Appendix~\ref{app:baseline-metrics}} details baselines, evaluation metrics (ECE, robustness metrics, selective prediction),
    and reporting conventions (seeds, fairness, confidence definitions).

    \item \textbf{Appendix~\ref{app:appendix_results}} reports additional experimental analyses:
    \begin{itemize}
        \item Threshold-based selective prediction (MNLI, SQuAD): \cref{app:mnli_squad_thresh,tab:5n2p-thresh}.
        \item OOD threshold-based selective prediction (HANS/ANLI/SNLI/IMDb): \cref{app:ood_thresh} and Table~\ref{tab:5n2p-ood-thresh}.
        \item Global/document-level MC Dropout diagnostics: \cref{app:global_mc} and Fig.\ref{fig:mc_diagnostics_heatmap}.
        \item HANS per-heuristic OOD breakdown: \cref{app:hans_heuristics,tab:5n2p-hans-heur}.
        \item Component ablations of uncertainty injection: \cref{app:component_ablation,fig:ablation}.
        \item Aggregation and normalization diagnostics (MNLI): \cref{app:agg_ln_diagnostics,tab:wxzc_checks_merged}.
        \item Scale and geometry diagnostics for token uncertainty: \cref{app:scale-geometry} and Table~\ref{tab:rescaling}, \ref{tab:corr} and \ref{tab:variants}.
        \item Negation diagnostics and qualitative attention illustration: \cref{app:negation,tab:negation,fig:uncertainty_attention}.
        \item Behavioral analysis by linguistic bucket: \cref{app:behavior_medqa} and Table~\ref{tab:behavior_medqa}.
        
        \item Sensitivity to $\lambda$ and $M$: \cref{app:sensitivity,tab:sensitivity}.
        \item Character-level adversarial robustness: \cref{app:appendix_adversarial,tab:adversarial}.
        \item Fine-grained linguistic uncertainty probes: \cref{app:full_linguistic_uncertainty,tab:linguistic}.
    \end{itemize}

    \item \textbf{Appendix~\ref{app:appendix_cross_model}} evaluates generalization across BERT-family backbones and model sizes,
    including cross-domain calibration and scaling effects.

    \item \textbf{Appendix~\ref{app:complexity-analysis}} documents implementation details, compute/latency overhead,
    hardware setup, and code availability for reproducibility.

    \item \textbf{Appendix~\ref{app:related}} provides an extended related work discussion.

    \item \textbf{Appendix~\ref{app:ai-usage}} provides an AI usage statement.
\end{itemize}

\newpage

\section{Mathematical Details}
\label{app:mathematical-details}

In this section, we present the formal mathematical foundations underlying the
uncertainty modeling and diagnostic analyses used in \ours.
We first introduce a layer-wise variance decomposition that attributes predictive
uncertainty across transformer depth under stochastic inference.
We then provide a proof of the decomposition, clarify its assumptions, and relate
Monte Carlo dropout to approximate Bayesian inference.
Finally, we discuss the practical and interpretive implications of layer-wise
uncertainty attribution for calibration, intervention, and uncertainty-aware
inference.
These analyses serve as diagnostic and explanatory tools that complement the
empirical results in the main paper, rather than as claims of exact posterior
inference.

\subsection{Layer-Wise Uncertainty Decomposition}
\label{app:layer-wise-uq-decomp}

To better understand how predictive uncertainty arises within the transformer,
we analyze its distribution across layers under stochastic inference.

\paragraph{Normalized variance contributions.}
To facilitate interpretability and ensure internal consistency, we report
\emph{normalized layer-wise variance contributions}, defined as
\begin{equation}
\tilde{\mathcal{V}}^{(l)} =
\frac{\mathcal{V}^{(l)}}{\sum_{k=1}^{L} \mathcal{V}^{(k)}},
\end{equation}
such that contributions across layers sum to $100\%$.
All percentages reported below correspond to these normalized values.

Table~\ref{tab:layer_uncertainty} presents representative layer-wise uncertainty
decompositions for answerable and unanswerable SQuAD examples, illustrating how
predictive uncertainty is distributed across transformer depth rather than
appearing only at the output layer.
We emphasize that these examples are illustrative diagnostics rather than
statistical summaries and are intended to provide qualitative insight into
uncertainty propagation rather than population-level guarantees.

For the \textbf{unanswerable example}, uncertainty grows progressively through
the self-attention stack.
While embedding-level contributions remain modest
(3.6\%), uncertainty increases substantially in mid and late layers, with the
largest contribution arising from layers 9--11 (12.5\%).
This pattern is
consistent with uncertainty emerging during later reasoning stages when
sufficient evidence is unavailable, rather than arising only at the final
decision layer.

In contrast, the \textbf{answerable example} exhibits a similar normalized
uncertainty distribution but a different internal structure.
Mid-layer
attention contributes more strongly to evidence aggregation (8.7\%), while
late-layer uncertainty reflects confident task-level reasoning rather than
unresolved ambiguity.

These cases illustrate that aggregate predictive variance alone is insufficient
to characterize model reliability.
Instead, the \emph{distribution of uncertainty
across layers} provides insight into whether uncertainty arises from ambiguous
inputs or from higher-level reasoning processes.

\begin{table*}[t]
\centering
\small
\renewcommand{\arraystretch}{1.15}
\setlength{\tabcolsep}{6pt}

\begin{tabular}{l l c c l}
\toprule

\rowcolor{tblLDHead}
\textbf{Layer} & \textbf{Component} & \textbf{Normalized Var.} & \textbf{\%} & \textbf{Interpretation} \\
\midrule

\rowcolor{tblLDUn}
\multicolumn{5}{c}{\textbf{SQuAD Unanswerable}} \\
\rowcolor{tblLDUn}\addlinespace[0.3em]

\rowcolor{tblLDUn}
Layer 0     & Embedding         & 0.036 & 3.6\%  & Word ambiguity \\
\rowcolor{tblLDUn}
Layer 1--4  & Attention (early) & 0.054 & 5.4\%  & Context formation \\
\rowcolor{tblLDUn}
Layer 5--8  & Attention (mid)   & 0.088 & 8.8\%  & Escalating uncertainty \\
\rowcolor{tblLDUn}
Layer 9--11 & Attention (late)  & 0.125 & 12.5\% & Uncertainty consolidation \\
\rowcolor{tblLDUn}
Layer 12    & Output            & 0.021 & 2.1\%  & Final decision \\

\rowcolor{tblLDUn}\addlinespace[0.3em]
\rowcolor{tblLDUnTot}
\textbf{Total} &  & \textbf{1.000} & \textbf{100\%} & Correctly uncertain \\
\midrule

\rowcolor{tblLDAns}
\multicolumn{5}{c}{\textbf{SQuAD Answerable}} \\
\rowcolor{tblLDAns}\addlinespace[0.3em]

\rowcolor{tblLDAns}
Layer 0     & Embedding         & 0.037 & 3.7\%  & Clear input \\
\rowcolor{tblLDAns}
Layer 1--4  & Attention (early) & 0.054 & 5.4\%  & Evidence aggregation \\
\rowcolor{tblLDAns}
Layer 5--8  & Attention (mid)   & 0.087 & 8.7\%  & Confidence reinforcement \\
\rowcolor{tblLDAns}
Layer 9--11 & Attention (late)  & 0.126 & 12.6\% & Confident reasoning \\
\rowcolor{tblLDAns}
Layer 12    & Output            & 0.023 & 2.3\%  & Final decision \\

\rowcolor{tblLDAns}\addlinespace[0.3em]
\rowcolor{tblLDAnsTot}
\textbf{Total} &  & \textbf{1.000} & \textbf{100\%} & Confident prediction \\

\bottomrule
\end{tabular}

\caption{Normalized layer-wise uncertainty decomposition for representative
answerable and unanswerable SQuAD examples (single seed).
Percentages indicate the relative contributions of each component to total predictive uncertainty.}
\label{tab:layer_uncertainty}
\end{table*}

\subsection{Derivation of Proposition~\ref{thm:layer_variance}}
\label{app:appendix-proof}

\begin{proof}
We derive a \emph{diagnostic} layer-wise variance decomposition by recursively
applying the law of total variance across transformer depth.
The purpose of this derivation is interpretive:
it provides a layer-indexed summary of where predictive variance arises
or amplifies during stochastic inference.

\paragraph{Setup and notation.}
Let $\hat{y}$ denote the model's final prediction (e.g., a class probability or logit),
and let $\mathbf{h}^{(l)} \in \mathbb{R}^{d}$ denote the hidden representation at
layer $l \in \{0, \ldots, L\}$ corresponding to the \texttt{[CLS]} token.
Stochasticity is introduced through Monte Carlo dropout at inference time.

Each transformer layer implements a stochastic mapping
\begin{equation}
\mathbf{h}^{(l)} = f^{(l)}\!\left(\mathbf{h}^{(l-1)}, \epsilon^{(l)}\right),
\qquad
\epsilon^{(l)} \sim \mathcal{D}^{(l)},
\end{equation}
where $\epsilon^{(l)}$ denotes dropout-induced noise at layer $l$.
We assume finite second moments so that all conditional variances below are well defined.

\paragraph{Recursive application of total variance.}
Applying the law of total variance at the final layer gives
\begin{equation}
\begin{aligned}
\operatorname{Var}(\hat{y})
&=
\mathbb{E}_{\mathbf{h}^{(L-1)}}
\!\left[
\operatorname{Var}\!\left(
\hat{y} \mid \mathbf{h}^{(L-1)}
\right)
\right] \\
&\quad+
\operatorname{Var}_{\mathbf{h}^{(L-1)}}
\!\left(
\mathbb{E}\!\left[
\hat{y} \mid \mathbf{h}^{(L-1)}
\right]
\right).
\end{aligned}
\end{equation}
The same identity can then be applied recursively to the second term by conditioning
on progressively earlier hidden representations
$\mathbf{h}^{(L-2)}, \mathbf{h}^{(L-3)}, \ldots, \mathbf{h}^{(0)}$.
This yields a recursive decomposition of predictive variance across depth.

\paragraph{Diagnostic layer-wise contributions.}
Motivated by this recursion, we define the layer-wise diagnostic contribution of layer $l$ as
\begin{equation}
\mathcal{V}^{(l)}
\triangleq
\mathbb{E}_{\mathbf{h}^{(l-1)}}
\left[
\operatorname{Var}_{\epsilon^{(l)}}
\!\left(
\hat{y} \mid \mathbf{h}^{(l-1)}
\right)
\right].
\end{equation}
These quantities summarize the predictive variability associated with the stochasticity
introduced at each layer, conditioned on the representation propagated from preceding layers.

\paragraph{Interpretation.}
Under this construction, predictive variance admits a recursive layer-indexed decomposition,
and the normalized values
\[
\tilde{\mathcal{V}}^{(l)}
=
\frac{\mathcal{V}^{(l)}}{\sum_{k=1}^{L}\mathcal{V}^{(k)}}
\]
provide an interpretable summary of how uncertainty is distributed across depth.
We emphasize that these contributions are intended as \emph{diagnostic attributions}:
they are useful for identifying where uncertainty concentrates or amplifies during inference,
but they should not be interpreted as unique causal effects or as a novel variance identity.
\end{proof}

\subsection{Bayesian Interpretation of MC Dropout}
\label{app:bayesian-interpretation}

Monte Carlo dropout provides a computationally efficient approximation to
Bayesian inference in deep neural networks \cite{gal2016dropout}.
Rather than maintaining a full posterior distribution over model parameters,
dropout induces stochasticity that can be interpreted as sampling from an
implicit variational posterior under standard assumptions.

\paragraph{Bayesian predictive distribution.}
In Bayesian inference, predictions marginalize over uncertainty in model
parameters:
\begin{equation}
p(y \mid \mathbf{x}, \mathcal{D})
=
\int p(y \mid \mathbf{x}, \theta)\, p(\theta \mid \mathcal{D}) \, d\theta.
\end{equation}

\paragraph{Monte Carlo approximation via dropout.}
MC dropout approximates this marginalization by retaining dropout during
inference and performing $M$ stochastic forward passes:
\begin{equation}
p(y \mid \mathbf{x}, \mathcal{D})
\approx
\frac{1}{M}
\sum_{m=1}^{M}
p(y \mid \mathbf{x}, \theta_m).
\end{equation}

\paragraph{Uncertainty interpretation.}
Variability across predictions reflects epistemic uncertainty arising from
model and data limitations and is, in principle, reducible with additional information.

\paragraph{Relevance to \ours.}
In \ours, these uncertainty estimates are explicitly
propagated through attention mechanisms, enabling uncertainty-aware inference
without modifying training objectives or learned parameters.

\subsection{Implications of Layer-Wise Uncertainty}
\label{app:layerwise-implications}

\paragraph{Attribution.}
Decomposing predictive variance across transformer depth enables identification
of where uncertainty is amplified during inference.
As shown in Table~\ref{tab:layer_uncertainty}, unanswerable SQuAD examples
concentrate substantial uncertainty in late layers (9--11), consistent with
ambiguity emerging during final reasoning stages.
Answerable examples exhibit comparable late-layer contributions, but the
uncertainty reflects confident task-level reasoning rather than unresolved
ambiguity.

\paragraph{Intervention.}
Layer-wise attribution suggests that calibration strategies could, in principle,
target high-variance components.
Our ablation study (\cref{fig:ablation}) shows that uncertainty introduced
at the attention level yields a \textbf{21.1\%} improvement in calibration, while
embedding-level uncertainty contributes negatively (\textbf{-3.1\%}),
supporting selective intervention rather than uniform stochasticity.

\paragraph{Interpretability.}
Unlike post-hoc calibration methods that produce a single uncertainty score,
this decomposition reveals how uncertainty evolves during reasoning.
For complex syntactic constructions, uncertainty is elevated in early and mid
layers due to structural processing, while late layers consolidate the final
decision.

\paragraph{Scope and limitation (per-layer recomputation).}
While we analyze \emph{where} uncertainty concentrates across depth via layer-wise variance attribution (Table~\ref{tab:layer_uncertainty}),
our \emph{UAT-LITE intervention} computes token uncertainty once (from stochastic forward passes) and reuses it across layers for efficiency.
Recomputing token uncertainty after each layer would require defining uncertainty over mixed representations after contextualization and would
increase inference cost roughly in proportion to the number of layers; we leave such layer-wise recomputation to future work.

\subsection{Practical Inference-Time Procedure}
\label{app:practical-instantiation}

At test time, Monte Carlo dropout produces multiple stochastic forward passes of
the same pretrained transformer.
These samples yield a predictive mean and
variance.
The mean is used as the final prediction, while the variance
quantifies epistemic uncertainty.

Uncertainty estimates influence inference by modulating attention under
representation ambiguity and enabling confidence-aware decision shaping at the
output level.
All uncertainty mechanisms operate strictly at inference time and
do not modify training dynamics or learned parameters.

\section{Inference-Time Algorithm}
\label{app:algorithm}

\paragraph{Inference-time algorithm.}
Algorithm~\ref{alg:uat-lite} summarizes the inference-time procedure used by \ours.
Uncertainty estimation and prediction are performed using the same $M$ stochastic
forward passes, without any separate embedding-only sampling stage.

\begin{algorithm}[t]
\caption{Inference-Time Uncertainty-Aware Attention (\ours)}
\label{alg:uat-lite}
\begin{algorithmic}[1]
\REQUIRE Input tokens $\mathbf{x}$, Monte Carlo budget $M$, uncertainty penalty $\lambda$
\STATE Initialize running statistics for token embeddings (per token $x_j$), with $\hat{U}^{(0)}(x_j)=0$
\FOR{$m = 1$ to $M$}
    \STATE Sample dropout masks
    \STATE Compute token embeddings $\mathbf{z}^{(m)}$
    \STATE Apply uncertainty-weighted attention using the lagged estimate $\hat{U}^{(m-1)}(x_j)$
    \STATE Compute logits $\ell^{(m)}$
    \STATE Update running estimates of token-level uncertainty $\hat{U}^{(m)}(x_j)$ from $\mathbf{z}^{(m)}$
\ENDFOR
\STATE Aggregate logits $\{\ell^{(m)}\}_{m=1}^M$ to obtain the predictive mean
\RETURN Prediction and uncertainty estimates
\end{algorithmic}
\end{algorithm}

\section{Baselines and Metrics}
\label{app:baseline-metrics}

We compare against BERT-base (110M parameters), MC Dropout ($p=0.1$, $M=5$), \TS, and Deep Ensembles (5 independently fine-tuned models).
\TS is treated as a strong post-hoc calibration baseline that
operates purely at the output level and does not modify internal model computation.
Because TS rescales logits without affecting representations, it is compatible
with uncertainty-aware inference and can be applied on top of \ours.
Accordingly, we report both \TS\ as a standalone baseline and the stacked variant
\ours+\TS\ to assess composability.
For \ours+\TS, the temperature $T$ is fit on the development set using the
MC-mean logits produced by \ours\ and applied at test time.
To isolate the effect of uncertainty-aware internal computation,
we compare \ours\ against non-stacked baselines (e.g., Base, MC Dropout, and \TS),
and treat \ours+\TS\ as an explicit compositional setting rather than the
core mechanism itself.
All baselines use the same pretrained backbone, fine-tuning recipe, and evaluation
code unless otherwise stated.

\paragraph{Bayesian baselines beyond inference-time methods.}
Methods such as SWAG and Laplace approximations estimate parameter-space uncertainty
through task-specific posterior fitting or fine-tuning \citep{talman2023uncertainty,ritter2018scalable}.
Because our scope is \emph{inference-time calibration} of pretrained transformers
without additional training or posterior fitting, we do not evaluate these approaches
as direct baselines.
We discuss them in Related Work (\cref{app:related}) as complementary directions.

\paragraph{Expected Calibration Error (ECE).}
Calibration is measured using Expected Calibration Error (ECE):
\begin{equation}
\label{eq:ece}
\mathrm{ECE}
=
\sum_{k=1}^{K}
\frac{|B_k|}{N}
\left|
\mathrm{acc}(B_k) - \mathrm{conf}(B_k)
\right|,
\end{equation}
where predictions are partitioned into $K=15$ \emph{fixed-width confidence bins}
$\{B_k\}$ over $[0,1]$.
For each bin, $\mathrm{acc}(B_k)$ denotes empirical accuracy, and
$\mathrm{conf}(B_k)$ denotes mean predicted confidence.
Here, $K$ denotes the number of confidence bins used for calibration evaluation,
which is distinct from the number of Monte Carlo samples $M$ used for uncertainty
estimation.
Lower ECE indicates better calibration.

\paragraph{Confidence definition.}
For classification tasks (MNLI, SST-2, MedQA), confidence is defined as the
maximum softmax probability of the predicted class, computed after aggregating
Monte Carlo samples.
For SQuAD~2.0, confidence corresponds to the predicted answerability probability
(unanswerable vs.\ answerable), following standard practice in uncertainty-aware
question answering.

\paragraph{SQuAD 2.0 evaluation protocol.}
We evaluate SQuAD~2.0 on the answerability prediction task, in which the model determines whether a question is answerable given the context.
This binary classification setting is well suited to calibration analysis.
We report answerability prediction accuracy alongside ECE in Table~\ref{tab:main_results_with_clinical}.
Span-level extraction (predicting answer start and end positions) is a separate task and is orthogonal to our uncertainty calibration contribution.
Our focus is on calibrated confidence scores for the answerability decision rather than on optimizing span extraction performance.

\paragraph{Distribution shift robustness metrics.}
To assess calibration stability under distributional shift, we evaluate models
on both in-domain (ID) and out-of-distribution (OOD) data and report two
complementary metrics.
First, we measure calibration drift:
\begin{equation}
\Delta\mathrm{ECE}
=
\mathrm{ECE}_{\text{OOD}} - \mathrm{ECE}_{\text{ID}},
\end{equation}
where smaller (or negative) values indicate greater robustness under shift.
Second, we report a calibration robustness score defined as the average ECE across
ID and OOD settings:
\begin{equation}
\mathrm{Robustness}
=
\frac{1}{2}
\left(
\mathrm{ECE}_{\text{ID}} + \mathrm{ECE}_{\text{OOD}}
\right),
\end{equation}
with lower values indicating more stable calibration across domains.

\paragraph{Selective prediction metrics.}
We additionally report selective coverage at fixed confidence thresholds and the
Area Under the Risk--Coverage Curve (AURC; lower is better).
Confidence thresholds are selected on the validation set \emph{for each method}
and then held fixed at test time.
We also report admitted counts ($N@\tau$) at fixed thresholds to make score-rescaling effects
(e.g., under TS) explicit.

\paragraph{MC Dropout baseline configuration.}
The MC Dropout baseline uses a uniform dropout rate of $p=0.1$ across all layers,
consistent with prior calibration studies.
Although more aggressive or component-specific dropout tuning could further improve
MC Dropout performance, such tuning increases implementation complexity and
reduces comparability with standard baselines.
In contrast, \ours explicitly leverages structured uncertainty
within the attention mechanism rather than relying solely on increased
stochasticity.
This comparison isolates the effect of uncertainty-aware attention rather than
raw dropout intensity.

\paragraph{Component-specific MC Dropout control.}
To disentangle the effect of structured stochasticity from uncertainty-aware
attention, we additionally evaluate an MC Dropout baseline using the same
component-specific dropout rates as \ours
(0.1 embeddings / 0.2 attention / 0.3 feed-forward), but \emph{without} attention
modulation.
As shown in Table~\ref{tab:main_results_with_clinical}, this control does not improve calibration
relative to uniform MC Dropout or the BERT baseline, indicating that the gains of
\ours arise from uncertainty-weighted attention rather than
tuned dropout alone.

\paragraph{Reporting conventions and fairness.}
Within each table, all methods are evaluated using the same data splits, evaluation code,
confidence definitions, and Monte Carlo budget $M$ (where applicable).
ECE is computed with $K=15$ fixed-width bins for all methods.
Temperature scaling (and \ours+TS) fits a single scalar temperature on
the validation set only and is applied at test time without access to test labels.
OOD calibration metrics (e.g., $\Delta$ECE and Robustness) are computed using the same
pipeline on ID and OOD sets to avoid binning or confidence-definition mismatch.

\paragraph{Random seeds and variability.}
\textbf{Main-paper results} (general calibration, distribution-shift robustness, and selective
prediction) are reported as mean $\pm$ standard deviation over five random seeds.
\textbf{Appendix diagnostic analyses}, including qualitative attention visualizations,
representative layer-wise uncertainty decompositions, and linguistic or adversarial
case studies, are reported using a single seed to reduce computational cost and are
explicitly labeled as \emph{single-seed} in the corresponding captions.
Deterministic baselines exhibit negligible variance because inference is fully deterministic
when evaluated from a fixed fine-tuned checkpoint reused across seeds.

\section{Additional Experimental Results}
\label{app:appendix_results}

This section presents additional analyses that complement the main experimental
results by providing finer-grained insight into the behavior of \ours.
Specifically, we examine selective prediction under fixed thresholds, calibration behavior under stronger OOD shifts,
and settings that reveal how uncertainty-aware attention responds to ambiguity,
structure, and semantic inconsistency.
Together, these analyses provide a more nuanced view of the strengths and
limitations of the proposed framework beyond aggregate calibration metrics.

\subsection{Additional Threshold-Based Selective Prediction Results}
\label{app:mnli_squad_thresh}

In addition to the broader OOD suites reported in the main text, we provide
threshold-based selective prediction results on MNLI (N=4{,}908) and SQuAD~2.0
(N=5{,}937) using the same evaluation protocol.
Confidence is defined as $s(x) = \max_y p_\theta(y \mid x)$, and for thresholds
$\tau \in \{0.9, 0.8, 0.7\}$ we report coverage (Cov@$\tau$), admitted accuracy
(Acc@$\tau$), admitted counts ($N@\tau$), and AURC.

Table~\ref{tab:5n2p-thresh} serves two purposes.
First, it shows that the selective prediction behavior observed under stronger OOD
conditions also appears in standard evaluation settings.
Second, it clarifies how post-hoc temperature scaling (TS) changes the admitted
fraction at fixed thresholds.
In particular, TS sharply compresses high-confidence regions
(e.g., Cov@0.9 on SQuAD), while maintaining high admitted accuracy.
This behavior reflects output probability rescaling rather than structural changes
in evidence aggregation.

By contrast, UAT-LITE modifies confidence through uncertainty-aware attention
modulation.
As a result, changes in coverage reflect altered internal evidence weighting rather
than purely scalar logit rescaling.
These results complement the OOD analysis in the main text by showing the same
protocol under milder distribution conditions.

\begin{table*}[t]
\centering
\setlength{\tabcolsep}{3pt}
\renewcommand{\arraystretch}{1.05}
\scriptsize

\begin{adjustbox}{max width=\linewidth, center}
\begin{tabular}{llccc ccc ccc c}
\toprule

\rowcolor{tblSPHeader}
\textbf{Dataset} & \textbf{Method} &
\multicolumn{3}{c}{$\tau=0.9$} &
\multicolumn{3}{c}{$\tau=0.8$} &
\multicolumn{3}{c}{$\tau=0.7$} &
\textbf{AURC$\downarrow$} \\

\cmidrule(lr){3-5}\cmidrule(lr){6-8}\cmidrule(lr){9-11}

\rowcolor{tblSPHeader}
& &
\multicolumn{1}{>{\columncolor{tblTauNine}}c}{\textbf{Cov}} &
\multicolumn{1}{>{\columncolor{tblTauNine}}c}{\textbf{Acc}} &
\multicolumn{1}{>{\columncolor{tblTauNine}}c}{\textbf{$N$}} &
\multicolumn{1}{>{\columncolor{tblTauEight}}c}{\textbf{Cov}} &
\multicolumn{1}{>{\columncolor{tblTauEight}}c}{\textbf{Acc}} &
\multicolumn{1}{>{\columncolor{tblTauEight}}c}{\textbf{$N$}} &
\multicolumn{1}{>{\columncolor{tblTauSeven}}c}{\textbf{Cov}} &
\multicolumn{1}{>{\columncolor{tblTauSeven}}c}{\textbf{Acc}} &
\multicolumn{1}{>{\columncolor{tblTauSeven}}c}{\textbf{$N$}} &
\\
\midrule

\multirow{5}{*}{MNLI}
& Baseline
\cellcolor{tblTauNine}\; & \cellcolor{tblTauNine}0.3857 & \cellcolor{tblTauNine}0.8822 & \cellcolor{tblTauNine}1893
\cellcolor{tblTauEight}\; & \cellcolor{tblTauEight}0.5591 & \cellcolor{tblTauEight}0.8455 & \cellcolor{tblTauEight}2744
\cellcolor{tblTauSeven}\; & \cellcolor{tblTauSeven}0.6999 & \cellcolor{tblTauSeven}0.8032 & \cellcolor{tblTauSeven}3435
& 0.1446 \\

& Base+TS
\cellcolor{tblTauNine}\; & \cellcolor{tblTauNine}0.1795 & \cellcolor{tblTauNine}0.9432 & \cellcolor{tblTauNine}881
\cellcolor{tblTauEight}\; & \cellcolor{tblTauEight}0.3820 & \cellcolor{tblTauEight}0.8853 & \cellcolor{tblTauEight}1875
\cellcolor{tblTauSeven}\; & \cellcolor{tblTauSeven}0.5515 & \cellcolor{tblTauSeven}0.8493 & \cellcolor{tblTauSeven}2707
& 0.1441 \\

& MC Dropout (M=10)
\cellcolor{tblTauNine}\; & \cellcolor{tblTauNine}0.3048 & \cellcolor{tblTauNine}0.8984 & \cellcolor{tblTauNine}1496
\cellcolor{tblTauEight}\; & \cellcolor{tblTauEight}0.4890 & \cellcolor{tblTauEight}0.8596 & \cellcolor{tblTauEight}2400
\cellcolor{tblTauSeven}\; & \cellcolor{tblTauSeven}0.6447 & \cellcolor{tblTauSeven}0.8186 & \cellcolor{tblTauSeven}3164
& 0.1478 \\

& UAT-LITE (M=10)
\cellcolor{tblTauNine}\; & \cellcolor{tblTauNine}0.3195 & \cellcolor{tblTauNine}0.8929 & \cellcolor{tblTauNine}1568
\cellcolor{tblTauEight}\; & \cellcolor{tblTauEight}0.5026 & \cellcolor{tblTauEight}0.8541 & \cellcolor{tblTauEight}2467
\cellcolor{tblTauSeven}\; & \cellcolor{tblTauSeven}0.6524 & \cellcolor{tblTauSeven}0.8154 & \cellcolor{tblTauSeven}3202
& 0.1483 \\

& UAT-LITE+TS
\cellcolor{tblTauNine}\; & \cellcolor{tblTauNine}0.1746 & \cellcolor{tblTauNine}0.9428 & \cellcolor{tblTauNine}857
\cellcolor{tblTauEight}\; & \cellcolor{tblTauEight}0.3714 & \cellcolor{tblTauEight}0.8826 & \cellcolor{tblTauEight}1823
\cellcolor{tblTauSeven}\; & \cellcolor{tblTauSeven}0.5383 & \cellcolor{tblTauSeven}0.8471 & \cellcolor{tblTauSeven}2642
& 0.1479 \\

\midrule

\multirow{5}{*}{SQuAD 2.0}
& Baseline
\cellcolor{tblTauNine}\; & \cellcolor{tblTauNine}0.4745 & \cellcolor{tblTauNine}0.7000 & \cellcolor{tblTauNine}2817
\cellcolor{tblTauEight}\; & \cellcolor{tblTauEight}0.6468 & \cellcolor{tblTauEight}0.6974 & \cellcolor{tblTauEight}3840
\cellcolor{tblTauSeven}\; & \cellcolor{tblTauSeven}0.7780 & \cellcolor{tblTauSeven}0.6850 & \cellcolor{tblTauSeven}4619
& 0.3186 \\

& Base+TS
\cellcolor{tblTauNine}\; & \cellcolor{tblTauNine}0.0000 & \cellcolor{tblTauNine}0.0000 & \cellcolor{tblTauNine}0
\cellcolor{tblTauEight}\; & \cellcolor{tblTauEight}0.0372 & \cellcolor{tblTauEight}0.6787 & \cellcolor{tblTauEight}221
\cellcolor{tblTauSeven}\; & \cellcolor{tblTauSeven}0.1816 & \cellcolor{tblTauSeven}0.6698 & \cellcolor{tblTauSeven}1078
& 0.3186 \\

& MC Dropout (M=10)
\cellcolor{tblTauNine}\; & \cellcolor{tblTauNine}0.3901 & \cellcolor{tblTauNine}0.7003 & \cellcolor{tblTauNine}2316
\cellcolor{tblTauEight}\; & \cellcolor{tblTauEight}0.5875 & \cellcolor{tblTauEight}0.6995 & \cellcolor{tblTauEight}3488
\cellcolor{tblTauSeven}\; & \cellcolor{tblTauSeven}0.7349 & \cellcolor{tblTauSeven}0.6890 & \cellcolor{tblTauSeven}4363
& 0.3150 \\

& UAT-LITE (M=10)
\cellcolor{tblTauNine}\; & \cellcolor{tblTauNine}0.4051 & \cellcolor{tblTauNine}0.6998 & \cellcolor{tblTauNine}2405
\cellcolor{tblTauEight}\; & \cellcolor{tblTauEight}0.5979 & \cellcolor{tblTauEight}0.7000 & \cellcolor{tblTauEight}3550
\cellcolor{tblTauSeven}\; & \cellcolor{tblTauSeven}0.7408 & \cellcolor{tblTauSeven}0.6885 & \cellcolor{tblTauSeven}4398
& 0.3177 \\

& UAT-LITE+TS
\cellcolor{tblTauNine}\; & \cellcolor{tblTauNine}0.0000 & \cellcolor{tblTauNine}0.0000 & \cellcolor{tblTauNine}0
\cellcolor{tblTauEight}\; & \cellcolor{tblTauEight}0.0487 & \cellcolor{tblTauEight}0.6574 & \cellcolor{tblTauEight}289
\cellcolor{tblTauSeven}\; & \cellcolor{tblTauSeven}0.1728 & \cellcolor{tblTauSeven}0.6676 & \cellcolor{tblTauSeven}1026
& 0.3177 \\

\bottomrule
\end{tabular}
\end{adjustbox}

\caption{\textbf{Threshold-based selective prediction on MNLI and SQuAD~2.0 under fixed confidence thresholds.}
Results are reported on MNLI ($N=4908$) and SQuAD~2.0 ($N=5937$) for $\tau\in\{0.9,0.8,0.7\}$.
Cov@$\tau$ denotes the admitted fraction, Acc@$\tau$ the accuracy on the admitted subset, and $N@\tau$ the admitted count; AURC summarizes the full risk--coverage trade-off (lower is better).
The table shows that temperature-scaled variants admit substantially fewer examples at high thresholds, especially on SQuAD~2.0, while UAT-LITE variants tend to preserve higher coverage at the same operating points.
These results provide a standard-setting counterpart to the broader OOD selective-prediction analyses reported elsewhere in the paper.}
\label{tab:5n2p-thresh}
\end{table*}

\subsection{OOD Threshold-Based Selective Prediction (HANS/ANLI/SNLI/IMDb)}
\label{app:ood_thresh}

To further evaluate uncertainty quality under stronger distribution shifts, we conduct
threshold-based selective prediction experiments across four widely used
out-of-distribution (OOD) evaluation suites: \textbf{HANS} ($N=30{,}000$),
\textbf{ANLI} ($N=3{,}200$), \textbf{SNLI} ($N=9{,}824$), and \textbf{IMDb} ($N=25{,}000$).
The full results are reported in Table~\ref{tab:5n2p-ood-thresh}.

Selective prediction evaluates whether a model can \emph{abstain} on low-confidence
examples while maintaining high accuracy on the retained subset.
For a confidence threshold $\tau \in \{0.9, 0.8, 0.7\}$, we admit predictions only when the
model confidence exceeds $\tau$.
For each threshold, we report coverage (Cov@$\tau$), admitted accuracy (Acc@$\tau$),
admitted count ($N@\tau$), and AURC.

Table~\ref{tab:5n2p-ood-thresh} consolidates results across all OOD suites in a single
view, with Panels A--D corresponding to HANS, ANLI, SNLI, and IMDb, respectively.
HANS and ANLI represent particularly challenging shifts designed to break common NLI heuristics,
whereas SNLI and IMDb provide additional cross-domain evaluation settings.

Across datasets, the table shows how different uncertainty estimation methods
affect selective prediction behavior.
In particular, TS often changes probability calibration without altering the predicted label, which can
substantially shift coverage at fixed thresholds.
Accordingly, differences in Cov@$\tau$ should be interpreted as changes in the \emph{confidence
distribution} rather than in classification decisions.
This appendix therefore complements the MNLI transfer analysis by providing a broader
view of uncertainty-guided filtering under diverse OOD conditions.

%

\begin{table*}[t]
\centering
\setlength{\tabcolsep}{3.0pt}
\renewcommand{\arraystretch}{1.06}
\scriptsize

\resizebox{\textwidth}{!}{%
\begin{tabular}{l c c c c c c c c c c c c c}
\toprule

\rowcolor{tblHansHeader}
\textbf{Method} &
\textbf{Acc} &
\textbf{ECE$\downarrow$} &
\textbf{NLL$\downarrow$} &
\textbf{Cov@.9} & \textbf{Acc@.9} & \textbf{N@.9} &
\textbf{Cov@.8} & \textbf{Acc@.8} & \textbf{N@.8} &
\textbf{Cov@.7} & \textbf{Acc@.7} & \textbf{N@.7} &
\textbf{AURC$\downarrow$} \\
\midrule

\rowcolor{tblHansSection}
\multicolumn{14}{l}{\textbf{Panel A: HANS OOD (N=30{,}000)}} \\
\midrule

\rowcolor{tblHansA}
Baseline      & 0.4808 & 0.3092 & 1.0581 & 0.2209 & 0.4831 &  6626 & 0.5427 & 0.4930 & 16281 & 0.7502 & 0.4840 & 22507 & 0.5287 \\
\rowcolor{tblHansA}
Base+TS       & 0.4836 & 0.2283 & 0.8737 & 0.0625 & 0.3952 &  1875 & 0.2333 & 0.4856 &  6999 & 0.5415 & 0.4942 & 16246 & 0.5268 \\
\rowcolor{tblHansA}
MC Dropout    & 0.4810 & 0.2506 & 0.9256 & 0.0883 & 0.4392 &  2650 & 0.3428 & 0.4816 & 10284 & 0.5929 & 0.4799 & 17788 & 0.5322 \\
\rowcolor{tblHansB}
UAT-LITE      & 0.4807 & 0.2395 & 0.9074 & 0.0687 & 0.4469 &  2061 & 0.3039 & 0.4650 &  9117 & 0.5583 & 0.4646 & 16749 & 0.5385 \\
\rowcolor{tblHansB}
UAT-LITE+TS   & 0.4817 & 0.1908 & 0.8222 & 0.0210 & 0.3386 &   629 & 0.1312 & 0.4554 &  3935 & 0.3951 & 0.4673 & 11853 & 0.5386 \\

\midrule

\rowcolor{tblHansSection}
\multicolumn{14}{l}{\textbf{Panel B: ANLI OOD (N=3{,}200)}} \\
\midrule

\rowcolor{tblHansA}
Baseline      & 0.3109 & 0.4209 & 1.9074 & 0.2050 & 0.2668 &   656 & 0.3978 & 0.2875 &  1273 & 0.5747 & 0.2953 &  1839 & 0.7130 \\
\rowcolor{tblHansA}
Base+TS       & 0.3109 & 0.3413 & 1.5450 & 0.0541 & 0.2370 &   173 & 0.2013 & 0.2609 &   644 & 0.3769 & 0.2886 &  1206 & 0.7126 \\
\rowcolor{tblHansA}
MC Dropout    & 0.3106 & 0.3916 & 1.7681 & 0.1456 & 0.2446 &   466 & 0.3231 & 0.2834 &  1034 & 0.5094 & 0.2914 &  1630 & 0.7141 \\
\rowcolor{tblHansB}
UAT-LITE      & 0.3091 & 0.4040 & 1.8222 & 0.1778 & 0.2566 &   569 & 0.3516 & 0.2836 &  1125 & 0.5319 & 0.2955 &  1702 & 0.7130 \\
\rowcolor{tblHansB}
UAT-LITE+TS   & 0.3091 & 0.3443 & 1.5591 & 0.0656 & 0.2381 &   210 & 0.2112 & 0.2692 &   676 & 0.3847 & 0.2851 &  1231 & 0.7124 \\

\midrule

\rowcolor{tblHansSection}
\multicolumn{14}{l}{\textbf{Panel C: SNLI OOD (N=9{,}824)}} \\
\midrule

\rowcolor{tblHansA}
Baseline      & 0.6810 & 0.0684 & 0.7757 & 0.2579 & 0.9013 &  2534 & 0.4540 & 0.8377 &  4460 & 0.6162 & 0.7907 &  6054 & 0.1772 \\
\rowcolor{tblHansA}
Base+TS       & 0.6810 & 0.0179 & 0.7497 & 0.0618 & 0.9176 &   607 & 0.2513 & 0.9032 &  2469 & 0.4379 & 0.8440 &  4302 & 0.1760 \\
\rowcolor{tblHansA}
MC Dropout    & 0.6819 & 0.0328 & 0.7497 & 0.1766 & 0.9182 &  1735 & 0.3727 & 0.8651 &  3661 & 0.5408 & 0.8095 &  5313 & 0.1752 \\
\rowcolor{tblHansB}
UAT-LITE      & 0.6759 & 0.0474 & 0.7650 & 0.1995 & 0.9122 &  1960 & 0.3954 & 0.8540 &  3884 & 0.5610 & 0.8044 &  5511 & 0.1792 \\
\rowcolor{tblHansB}
UAT-LITE+TS   & 0.6759 & 0.0195 & 0.7534 & 0.0693 & 0.9383 &   681 & 0.2438 & 0.8944 &  2395 & 0.4274 & 0.8445 &  4199 & 0.1783 \\

\midrule

\rowcolor{tblHansSection}
\multicolumn{14}{l}{\textbf{Panel D: IMDb OOD (N=25{,}000)}} \\
\midrule

\rowcolor{tblHansA}
Baseline      & 0.8281 & 0.1000 & 0.4904 & 0.7780 & 0.8925 & 19451 & 0.8641 & 0.8687 & 21602 & 0.9173 & 0.8538 & 22932 & 0.0640 \\
\rowcolor{tblHansA}
Base+TS       & 0.8281 & 0.0442 & 0.3959 & 0.5960 & 0.9325 & 14899 & 0.7683 & 0.8950 & 19208 & 0.8634 & 0.8688 & 21586 & 0.0640 \\
\rowcolor{tblHansA}
MC Dropout    & 0.8195 & 0.0957 & 0.4845 & 0.7367 & 0.8938 & 18417 & 0.8405 & 0.8673 & 21012 & 0.9014 & 0.8495 & 22536 & 0.0695 \\
\rowcolor{tblHansB}
UAT-LITE      & 0.8186 & 0.1054 & 0.5078 & 0.7674 & 0.8862 & 19184 & 0.8570 & 0.8626 & 21426 & 0.9109 & 0.8463 & 22772 & 0.0694 \\
\rowcolor{tblHansB}
UAT-LITE+TS   & 0.8186 & 0.0501 & 0.4117 & 0.5781 & 0.9279 & 14453 & 0.7598 & 0.8882 & 18995 & 0.8583 & 0.8621 & 21457 & 0.0694 \\

\bottomrule
\end{tabular}}
\caption{\textbf{OOD threshold-based selective prediction across four evaluation suites.}
We report overall accuracy, ECE, NLL, Coverage@$\tau$, Acc@$\tau$, admitted counts ($N@\tau$) for $\tau \in \{0.9, 0.8, 0.7\}$, and AURC.
Panels A--D correspond to HANS ($N=30{,}000$), ANLI ($N=3{,}200$), SNLI ($N=9{,}824$), and IMDb ($N=25{,}000$), respectively.
Across datasets, temperature-scaled variants generally reduce ECE and NLL, but often admit substantially fewer examples at high confidence thresholds, especially on HANS and ANLI.
By contrast, UAT-LITE variants typically preserve higher coverage at the same thresholds, reflecting a different confidence profile under shift.
The table therefore highlights the trade-off between post-hoc confidence compression and uncertainty-aware filtering under diverse OOD conditions.}
\label{tab:5n2p-ood-thresh}
\vspace{-2mm}
\end{table*}

\subsection{Global Document-Level MC Dropout Diagnostics}
\label{app:global_mc}

To contextualize the behavior of uncertainty-weighted attention,
we evaluate standard \emph{global/document-level MC Dropout}
applied to the same BERT-base checkpoint used in the main results.
Here, dropout is injected at inference time without modifying
the attention mechanism, and predictive uncertainty is estimated
from Monte Carlo sampling ($M=10$).
\cref{fig:mc_diagnostics_heatmap} summarizes calibration (ECE, Brier, NLL),
selective prediction (AURC, Acc@70/80/90), and accuracy across SQuAD 2.0, MNLI, and SST-2
(absolute values shown in each cell; color denotes row-normalized relative performance).

The purpose of this analysis is diagnostic rather than competitive.
Global MC Dropout improves over the unscaled baseline in calibration
(e.g., ECE decreases across datasets), but it remains materially behind
temperature scaling in marginal calibration (see the TS gaps reported elsewhere),
confirming that naive stochastic inference alone does not recover TS-level calibration.

Importantly, this comparison isolates the effect of \emph{where} uncertainty is injected.
Unlike global MC Dropout which affects output variance while leaving evidence aggregation unchanged,
UAT-LITE incorporates uncertainty directly into the attention logits,
modifying evidence aggregation rather than only output variability.
The diagnostic therefore clarifies that improvements observed in the main text are not attributable
merely to Monte Carlo sampling, but to attention-level modulation.

\begin{figure}[t]
  \centering
  \includegraphics[width=\linewidth]{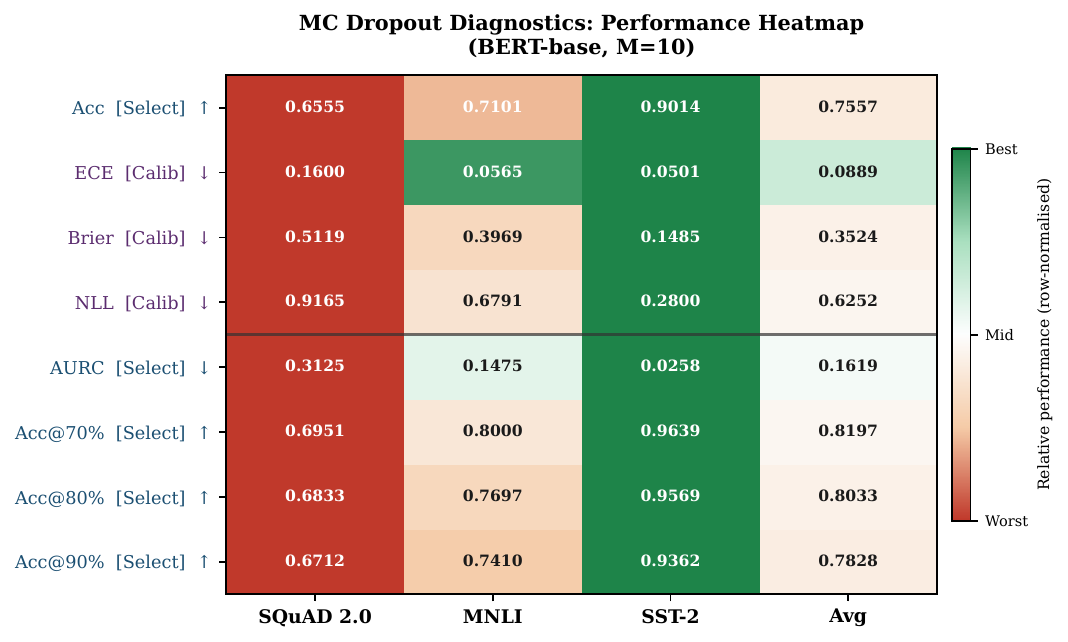}
  \caption{\textbf{Global/document-level MC Dropout diagnostics on BERT-base ($M=10$).}
  Heatmap of absolute performance for global MC Dropout across SQuAD 2.0, MNLI, and SST-2
  (metrics: Acc$\uparrow$, ECE$\downarrow$, Brier$\downarrow$, NLL$\downarrow$, AURC$\downarrow$, and Acc@70/80/90$\uparrow$).
  Each cell reports the numeric value; shading indicates row-normalized relative performance (best-to-worst within each metric row).
  While global MC Dropout improves calibration relative to the unscaled baseline, it remains substantially behind temperature scaling,
  motivating attention-level uncertainty modulation in UAT-LITE.}
  \label{fig:mc_diagnostics_heatmap}
  \vspace{-2mm}
\end{figure}

\subsection{Heuristic Sensitivity Analysis on HANS}
\label{app:hans_heuristics}

To localize failure modes under distribution shift, we provide a per-heuristic
breakdown on the HANS benchmark.
HANS isolates three common heuristic shortcuts: lexical overlap, subsequence,
and constituent structure.
Table~\ref{tab:5n2p-hans-heur} reports overall accuracy by heuristic category
(top block) and non-entailment-only accuracy (bottom block), which is the more
challenging and shortcut-prone regime.

The top block shows that overall per-heuristic differences are relatively modest
across methods.
However, the non-entailment-only breakdown is more revealing, because this subset
isolates cases where shallow heuristics often produce overconfident errors.
In this regime, UAT-LITE and especially UAT-LITE+TS show clear gains over the
baseline, with the largest improvements on constituent-based examples.
These results suggest that uncertainty-aware attention is most useful not in
easy heuristic-aligned cases, but in examples where shortcut-driven evidence is
misleading and the model must suppress superficially plausible matches.
Overall, this diagnostic provides a more localized view of robustness under shift
than aggregate HANS accuracy alone.

\begin{table*}[t]
\centering
\setlength{\tabcolsep}{10pt}
\renewcommand{\arraystretch}{1.12}
\small

\begin{tabular}{lccc}
\toprule
\rowcolor{tblHeurHeader}
\textbf{Method} & \textbf{Lex.\ Overlap} & \textbf{Subsequence} & \textbf{Constituent} \\
\midrule

\rowcolor{tblHeurTop}
Baseline      & 0.5043 & 0.4632 & 0.4750 \\
\rowcolor{tblHeurTop}
Base+TS       & \textbf{0.5104} & 0.4622 & 0.4783 \\
\rowcolor{tblHeurTop}
MC Dropout    & 0.5000 & 0.4642 & 0.4787 \\
\rowcolor{tblHeurTop}
UAT-LITE      & 0.4835 & \textbf{0.4777} & 0.4810 \\
\rowcolor{tblHeurTop}
UAT-LITE+TS   & 0.4859 & 0.4710 & \textbf{0.4881} \\

\midrule
\rowcolor{tblHeurHeader}
\multicolumn{4}{c}{\textit{Non-entailment only}} \\
\midrule

\rowcolor{tblHeurBot}
Baseline      & 0.1804 & 0.2358 & 0.3428 \\
\rowcolor{tblHeurBot}
Base+TS       & 0.2230 & 0.2714 & 0.3824 \\
\rowcolor{tblHeurBot}
MC Dropout    & 0.2132 & 0.2408 & 0.4158 \\
\rowcolor{tblHeurBot}
UAT-LITE      & 0.2960 & 0.2722 & 0.5580 \\
\rowcolor{tblHeurBot}
UAT-LITE+TS   & \textbf{0.3548} & \textbf{0.3184} & \textbf{0.6100} \\

\bottomrule
\end{tabular}

\caption{\textbf{HANS per-heuristic breakdown under distribution shift.}
The top block reports overall accuracy for the three HANS heuristic categories: lexical overlap, subsequence, and constituent structure.
The bottom block isolates \textbf{non-entailment-only} examples, which form the main shortcut-sensitive failure regime in HANS.
Overall per-heuristic accuracy differences are modest, but the non-entailment subset reveals clearer separation between methods.
UAT-LITE and especially UAT-LITE+TS improve substantially on this subset, with the largest gains on constituent examples, indicating better behavior when superficial heuristic matches are misleading.}
\label{tab:5n2p-hans-heur}
\end{table*}

\subsection{Component Ablation of Uncertainty Injection}
\label{app:component_ablation}

Fig.~\ref{fig:ablation} evaluates the contribution of
individual components of UAT-LITE on SQuAD~2.0 (validation). We isolate: (i) baseline inference, (ii) MC sampling without attention modulation, (iii) uncertainty-weighted attention,
and (iv) the full combined model. The primary metric is Expected Calibration Error (ECE), with accuracy reported for reference. Two observations are noteworthy. First, uncertainty-weighted attention produces the largest
reduction in ECE relative to baseline,
indicating that incorporating epistemic variance
into the attention logits directly improves calibration.

Second, combining sampling with attention modulation
yields further improvements,
suggesting that stochastic sampling and
structural uncertainty injection are complementary. Notably, accuracy remains stable across variants,
indicating that improvements in calibration
are not driven by trivial accuracy trade-offs. This ablation confirms that gains observed in the main text
are not attributable to Monte Carlo sampling alone,
but specifically to attention-level uncertainty weighting.

\begin{figure*}[t]
    \centering
    \includegraphics[width=\textwidth]{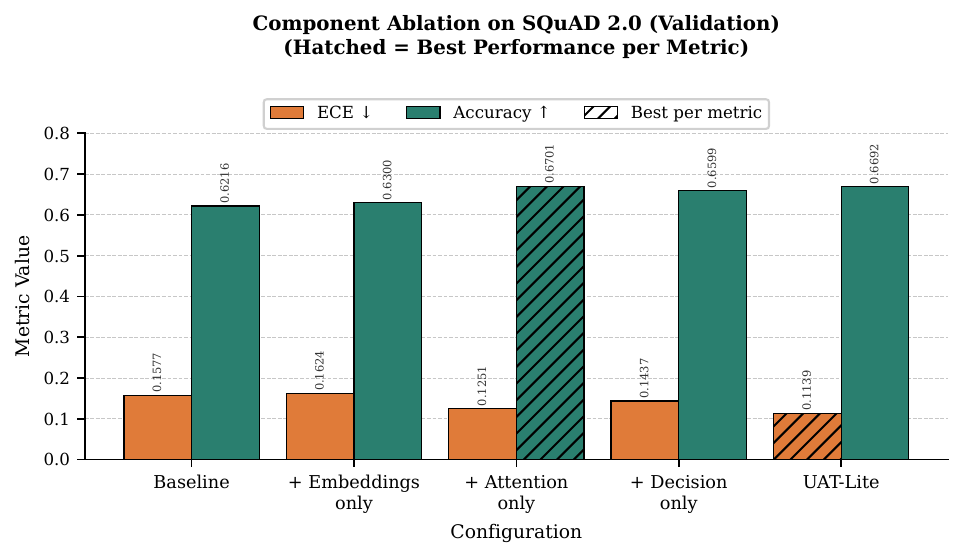}
    \caption{
    \textbf{Component ablation on SQuAD~2.0 (validation).}
    We compare baseline inference, embeddings-only stochasticity, attention-only uncertainty injection, decision-only uncertainty shaping, and the full UAT-LITE model.
    Expected Calibration Error (ECE; lower is better) is the primary metric, while accuracy is reported for reference.
    The attention-only variant yields the largest single-component ECE reduction relative to the baseline, indicating that uncertainty-weighted attention is the main contributor to the calibration gains.
    The full model achieves the best overall calibration among the shown configurations while maintaining accuracy comparable to the other variants.
    Hatched bars indicate the best value for each metric.
    }
    \label{fig:ablation}
\end{figure*}

\subsection{Penalty Functional Form Ablation}
\label{app:penalty_form}

To assess whether the exponential uncertainty penalty is essential, we compare three modulation forms under identical settings ($M=10$, $\lambda=1.0$, SQuAD 2.0, single seed):  
(i) exponential: $a_{ij}\exp(-\lambda u_{ij})$,  
(ii) linear (clipped): $a_{ij}\max(0, 1-\lambda u_{ij})$, and  
(iii) reciprocal: $a_{ij}/(1+\lambda u_{ij})$.

Fig.~\ref{fig:penalty_ablation} reports the results. The three forms yield broadly comparable accuracy and calibration behavior, indicating that the overall effect stems from uncertainty-aware modulation rather than a single functional choice. We retain the exponential form due to its smoothness, guaranteed positivity (without clipping), and interpretation as an additive logit penalty ($\log \tilde a_{ij} = \log a_{ij} - \lambda u_{ij}$), which offers a convenient and numerically well-behaved mechanism for downweighting high-uncertainty evidence.

\begin{figure}[h]
  \centering
  \includegraphics[width=\linewidth]{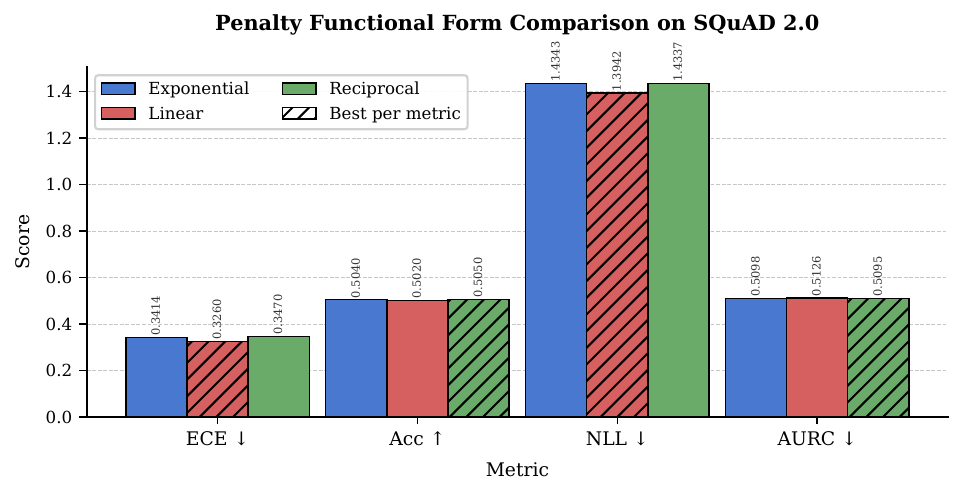}
  \caption{\textbf{Penalty functional form ablation (SQuAD 2.0; $M=10$, $\lambda=1.0$, single seed).}
  We compare three uncertainty-aware modulations of the attention weight $a_{ij}$:
  exponential $a_{ij}\exp(-\lambda u_{ij})$, linear (clipped) $a_{ij}\max(0,1-\lambda u_{ij})$, and reciprocal $a_{ij}/(1+\lambda u_{ij})$.
  Across metrics (ECE$\downarrow$, Acc$\uparrow$, NLL$\downarrow$, AURC$\downarrow$), all three yield similar performance, indicating the benefit comes from \emph{uncertainty-aware downweighting} rather than a specific functional choice. We use the exponential form as the default since it is smooth, strictly positive (no clipping), and equivalent to an additive logit penalty: $\log \tilde a_{ij}=\log a_{ij}-\lambda u_{ij}$.}
  \label{fig:penalty_ablation}
\end{figure}


\subsection{Aggregation and Normalization Diagnostics}
\label{app:agg_ln_diagnostics}

This section probes the robustness of two implementation choices underlying \ours:
(i) how we aggregate Monte Carlo variance across draws, and (ii) where normalization
is applied prior to uncertainty measurement.
We evaluate these choices using the MNLI negation diagnostic, which perturbs inputs
without changing the original labels, and we track both label-dependent signals
(error change) and label-free stability signals (confidence change) to ensure that
observed behaviors are not artifacts of a particular aggregation heuristic or a
specific LayerNorm placement.

\paragraph{Aggregation-statistic sanity check (no LayerNorm).}
Table~\ref{tab:wxzc_checks_merged} (top block) compares three variance aggregation
statistics (mean-std, RMS-std, median-std) on the MNLI negation diagnostic
($N{=}200$, seed$=42$, $\lambda{=}0.5$).
We report changes in confidence and error (dConf/dErr), computed as (negated $-$ original),
along with ECE on the sampled subset (SmpECE) and development ECE (DevECE).
Across aggregation schemes, dConf and dErr are tightly clustered, indicating that the
negation sensitivity signal is stable to the aggregation choice.
Although SmpECE/DevECE vary slightly across statistics, the qualitative conclusions
and effect directions remain unchanged.

\paragraph{LayerNorm invariance.}
Table~\ref{tab:wxzc_checks_merged} (bottom block) tests whether applying LayerNorm before
uncertainty measurement changes the diagnostic outcomes.
For each aggregation statistic, we compare dConf, dErr, and DevECE with and without
LayerNorm (``+LN'').
We observe nearly unchanged dConf and comparable dErr across normalization conditions,
while DevECE improves consistently under +LN.
Overall, these results suggest that \ours is not sensitive to the precise LayerNorm
placement and that the reported negation behaviors are not normalization artifacts.

\begin{table*}[t]
\centering
\setlength{\tabcolsep}{3.5pt}
\renewcommand{\arraystretch}{1.08}
\scriptsize

\begin{adjustbox}{max width=\textwidth, center}
\begin{tabular}{l c c c c c c}
\toprule

\rowcolor{tblMergeHead}
\textbf{Variant} &
\textbf{AccOrig} &
\textbf{AccNeg} &
\textbf{dConf} &
\textbf{dErr} &
\textbf{SmpECE} &
\textbf{DevECE} \\
\midrule

\rowcolor{tblMergeA}
\multicolumn{7}{c}{\textbf{Aggregation-statistic sanity check (no LayerNorm)}} \\
\rowcolor{tblMergeA}
Base Det.   & 0.7050 & 0.4950 & $-0.0831$ & $+0.2100$ & 0.1313 & -- \\
\rowcolor{tblMergeA}
mean-std    & 0.7000 & 0.4900 & $-0.0778$ & $+0.2100$ & 0.1123 & 0.0595 \\
\rowcolor{tblMergeA}
RMS-std     & 0.6800 & 0.4750 & $-0.0817$ & $+0.2050$ & 0.1199 & 0.0576 \\
\rowcolor{tblMergeA}
median-std  & 0.6850 & 0.5000 & $-0.0821$ & $+0.1850$ & 0.1056 & 0.0584 \\

\midrule

\rowcolor{tblMergeB}
\multicolumn{7}{c}{\textbf{LayerNorm-invariance check (with/without LN)}} \\
\rowcolor{tblMergeHead}
\textbf{Variant} &
\multicolumn{1}{c}{\textbf{dConf}} &
\multicolumn{1}{c}{\textbf{dConf+LN}} &
\multicolumn{1}{c}{\textbf{dErr}} &
\multicolumn{1}{c}{\textbf{dErr+LN}} &
\multicolumn{1}{c}{\textbf{DevECE}} &
\multicolumn{1}{c}{\textbf{DevECE+LN}} \\
\midrule

\rowcolor{tblMergeB}
mean-std   & $-0.0778$ & $-0.0782$ & $+0.2100$ & $+0.2000$ & 0.0595 & 0.0466 \\
\rowcolor{tblMergeB}
RMS-std    & $-0.0817$ & $-0.0816$ & $+0.2050$ & $+0.2150$ & 0.0576 & 0.0455 \\
\rowcolor{tblMergeB}
median-std & $-0.0821$ & $-0.0818$ & $+0.1850$ & $+0.2150$ & 0.0584 & 0.0482 \\

\bottomrule
\end{tabular}
\end{adjustbox}

\caption{Aggregation-statistic sanity check and LayerNorm-invariance check on MNLI.
Sanity check uses the MNLI negation diagnostic ($N{=}200$, seed$=42$, $\lambda{=}0.5$), reporting AccOrig/AccNeg and changes (dConf/dErr) under negation.
``SmpECE'' is ECE on the sampled subset; ``DevECE'' is dev ECE.
LayerNorm-invariance reports the same aggregation statistics with/without applying LayerNorm before uncertainty measurement; LayerNorm yields nearly unchanged dConf and comparable dErr, while improving dev ECE across aggregations.}
\label{tab:wxzc_checks_merged}
\vspace{-2mm}
\end{table*}


\subsection{Scale and Geometry Diagnostics for Token Uncertainty}

\label{app:scale-geometry}

This appendix section examines whether token-level uncertainty computed from embedding-space stochasticity reflects a meaningful epistemic signal or could be explained by simple geometric artifacts (e.g., global embedding rescaling, anisotropy, or magnitude effects). To stress-test this possibility, we conduct three complementary diagnostics on representative tasks (MNLI and SQuAD) under the same inference-time configuration used in the main experiments. First, we test robustness under explicit embedding rescaling ($e_j \leftarrow c\,e_j$) and quantify the resulting changes in task metrics and token-uncertainty rankings. Second, we evaluate alignment between token uncertainty and output-space epistemic quantities derived from MC dropout (predictive entropy and mutual information), while controlling for embedding norm to rule out trivial magnitude explanations. Third, we ablate alternative uncertainty definitions to verify that conclusions do not depend on a single aggregation statistic. Together, these analyses clarify what aspects of $U(x_j)$ are robust, what is scale-sensitive, and why the observed behavior is not reducible to embedding magnitude alone.

\subsubsection{Embedding-Rescaling Invariance Test}
This diagnostic tests whether token uncertainty derived from embedding-space stochasticity could be explained by simple embedding-scale effects. We rescale the input token embeddings at inference time by a constant factor $c\in\{0.5,1,2,4\}$ while keeping all other settings fixed, and rerun evaluation on MNLI and SQuAD. Table~\ref{tab:rescaling} reports task-level metrics and deltas relative to the unscaled baseline ($c{=}1$), along with (i) token-uncertainty rank stability via Spearman correlation $\rho(U,U_c)$ and (ii) attention-pattern stability via cosine similarity between the resulting attention weights.

Two observations are salient. First, attention modulation remains highly stable under rescaling (Attn CosSim $\ge 0.979$ across all $c$), indicating that the induced attention patterns are largely preserved even under substantial changes in embedding magnitude. Second, token-uncertainty rankings are not strictly scale-invariant ($\rho\approx 0.10$--$0.23$ for $c\neq 1$), and calibration metrics can shift under extreme rescaling (e.g., MNLI $\Delta$ECE up to $+0.0095$ at $c{=}2$, while SQuAD ECE improves at $c{=}4$). Thus, while the attention behavior is robust, the uncertainty signal exhibits bounded scale sensitivity, motivating the complementary sanity checks and uncertainty-definition ablations below.

\subsubsection{Epistemic Alignment via Correlation Checks}
To test whether token-level uncertainty $U(x_j)$ aligns with more standard output-space epistemic indicators, we compute (under the same MC dropout samples) predictive entropy $H[\bar p(y|x)]$ and mutual information $\mathrm{MI}=H[\bar p(y|x)]-\mathbb{E}[H[p(y|x,\omega)]]$, and measure their association with $U(x_j)$. Table~\ref{tab:corr} reports both Pearson and Spearman correlations. In addition, to rule out the possibility that $U(x_j)$ is simply tracking embedding magnitude, we report correlations with embedding norm $\|e_j\|$ and partial correlations between $U$ and $\{H,\mathrm{MI}\}$ controlling for $\|e_j\|$.

On MNLI, $U(x_j)$ shows consistent monotonic association with output-space uncertainty (Spearman $\approx 0.37$ with entropy and $\approx 0.41$ with MI), while linear correlations are small but statistically reliable given the large number of tokens. On SQuAD, correlations are weaker, suggesting task-dependent alignment between token- and output-space uncertainty. Importantly, $U(x_j)$ is essentially uncorrelated with embedding norm on both datasets (Pearson near $0$; Spearman near $0$), and the norm-controlled partial correlations remain nearly unchanged, indicating that the signal is not explained by embedding magnitude alone.

\subsubsection{Alternative Uncertainty Definitions: Robustness Ablation}
Finally, we evaluate whether the results depend on a particular choice of uncertainty definition. Beyond the default \texttt{mean-std}, we compare alternatives that are commonly used to summarize multivariate dispersion under stochastic embeddings: \texttt{cov-trace} (trace of the covariance), \texttt{cv} (coefficient of variation), and \texttt{norm-cv} (a normalized CV variant). Table~\ref{tab:variants} shows that performance and calibration remain stable across these variants on both MNLI and SQuAD (differences are small and do not change the qualitative conclusions), suggesting that UAT-LITE does not hinge on a single arbitrary aggregation of embedding-space variability.

\begin{table*}[t]
\centering
\small
\setlength{\tabcolsep}{6pt}

\begin{tabular}{lcccccccc}
\toprule
\rowcolor{tblHansHeader}
Dataset & $c$ & ECE & NLL & Acc & AURC & $\Delta$ECE & $\Delta$NLL & $\Delta$Acc \\
\midrule
\rowcolor{tblHansSection}
\multicolumn{9}{l}{\textbf{MNLI}} \\
\rowcolor{tblHansA}
MNLI & 0.5 & 0.0117 & 0.5335 & 0.7773 & 0.0836 & -0.0048 & +0.0138 & -0.0100 \\
\rowcolor{tblHansB}
MNLI & 1.0 & 0.0165 & 0.5197 & 0.7873 & 0.0783 & +0.0000 & +0.0000 & +0.0000 \\
\rowcolor{tblHansA}
MNLI & 2.0 & 0.0260 & 0.5418 & 0.7753 & 0.0864 & +0.0095 & +0.0221 & -0.0120 \\
\rowcolor{tblHansB}
MNLI & 4.0 & 0.0212 & 0.5553 & 0.7613 & 0.0924 & +0.0047 & +0.0356 & -0.0260 \\
\midrule
\rowcolor{tblHansSection}
\multicolumn{9}{l}{\textbf{SQuAD}} \\
\rowcolor{tblHansA}
SQuAD & 0.5 & 0.3216 & 1.3102 & 0.5075 & 0.4923 & -0.0123 & -0.0537 & +0.0035 \\
\rowcolor{tblHansB}
SQuAD & 1.0 & 0.3339 & 1.3639 & 0.5040 & 0.4913 & +0.0000 & +0.0000 & +0.0000 \\
\rowcolor{tblHansA}
SQuAD & 2.0 & 0.3250 & 1.2799 & 0.5020 & 0.4944 & -0.0089 & -0.0840 & -0.0020 \\
\rowcolor{tblHansB}
SQuAD & 4.0 & 0.2944 & 1.1441 & 0.5060 & 0.4937 & -0.0395 & -0.2198 & +0.0020 \\
\bottomrule
\end{tabular}

\vspace{4pt}

\begin{tabular}{lccc}
\toprule
\rowcolor{tblHansHeader}
Dataset & $c$ & Spearman $\rho(U, U_c)$ & Attn CosSim \\
\midrule
\rowcolor{tblHansSection}
\multicolumn{4}{l}{\textbf{MNLI}} \\
\rowcolor{tblHansA}
MNLI & 0.5 & 0.2332 & 0.9869 \\
\rowcolor{tblHansB}
MNLI & 2.0 & 0.2120 & 0.9842 \\
\rowcolor{tblHansA}
MNLI & 4.0 & 0.2015 & 0.9792 \\
\midrule
\rowcolor{tblHansSection}
\multicolumn{4}{l}{\textbf{SQuAD}} \\
\rowcolor{tblHansA}
SQuAD & 0.5 & 0.1522 & 0.9922 \\
\rowcolor{tblHansB}
SQuAD & 2.0 & 0.0998 & 0.9913 \\
\rowcolor{tblHansA}
SQuAD & 4.0 & 0.1245 & 0.9851 \\
\bottomrule
\end{tabular}

\caption{\textbf{Embedding-rescaling invariance diagnostic.}
We rescale input embeddings by a factor $c$ at inference time to test whether the proposed uncertainty signal is primarily a trivial artifact of embedding magnitude.
The \textbf{top panel} reports task-level performance on MNLI and SQuAD, including ECE, NLL, accuracy, AURC, and the corresponding changes relative to the reference setting $c{=}1$.
The \textbf{bottom panel} reports token-level uncertainty rank stability under rescaling via Spearman $\rho(U,U_c)$, together with attention-pattern stability measured by cosine similarity.
Across both datasets, attention patterns remain highly similar under rescaling, while uncertainty rankings remain positively correlated and task-level behavior changes only moderately.
Taken together, these results suggest that the method is not explained solely by raw embedding scale and remains reasonably stable under embedding rescaling.}
\label{tab:rescaling}
\end{table*}

\begin{table*}[t]
\centering
\small
\setlength{\tabcolsep}{4pt}
\renewcommand{\arraystretch}{1.05}
\resizebox{\textwidth}{!}{%
\begin{tabular}{lcccccccc}
\toprule
\rowcolor{tblGreenHeader}
Dataset &
$U$--$H$ (P) & $U$--$H$ (S) &
$U$--MI (P) & $U$--MI (S) &
$U$--$\|e\|$ (P) & $U$--$\|e\|$ (S) &
$\rho(U,H \mid \|e\|)$ & $\rho(U,\mathrm{MI} \mid \|e\|)$ \\
\midrule
\rowcolor{tblGreenA}
MNLI  & 0.0580 & 0.3677 & 0.0759 & 0.4137 & -0.0071 & -0.0236 & 0.0583 & 0.0759 \\
\rowcolor{tblGreenB}
SQuAD & -0.0163 & -0.0222 & 0.0055 & 0.1694 & -0.0256 & -0.0510 & -0.0157 & 0.0041 \\
\bottomrule
\end{tabular}%
}
\caption{Correlation sanity checks for token uncertainty $U(x_j)$ against output-space epistemic measures (predictive entropy $H$ and mutual information MI) computed from MC dropout. We also report correlations with embedding norm $\|e_j\|$ and partial correlations controlling for $\|e_j\|$. Token uncertainty is essentially uncorrelated with embedding norm, and norm-controlled partial correlations remain unchanged.}
\label{tab:corr}
\end{table*}

\begin{table}[t]
\centering
\small
\setlength{\tabcolsep}{6pt}
\begin{tabular}{l l c c c c}
\toprule
\rowcolor{tblLavHeader}
Dataset & $U(x_j)$ Variant & ECE & NLL & Acc & AURC \\
\midrule
\rowcolor{tblLavSection}
\multicolumn{6}{l}{\textbf{MNLI}} \\
\rowcolor{tblLavA}
MNLI  & mean-std  & 0.0158 & 0.5128 & 0.7908 & 0.0775 \\
\rowcolor{tblLavB}
MNLI  & cov-trace & 0.0158 & 0.5128 & 0.7908 & 0.0775 \\
\rowcolor{tblLavA}
MNLI  & cv        & 0.0165 & 0.5197 & 0.7873 & 0.0783 \\
\rowcolor{tblLavB}
MNLI  & norm-cv   & 0.0158 & 0.5128 & 0.7908 & 0.0775 \\
\midrule
\rowcolor{tblLavSection}
\multicolumn{6}{l}{\textbf{SQuAD}} \\
\rowcolor{tblLavA}
SQuAD & mean-std  & 0.3377 & 1.3866 & 0.5065 & 0.4940 \\
\rowcolor{tblLavB}
SQuAD & cov-trace & 0.3377 & 1.3866 & 0.5065 & 0.4940 \\
\rowcolor{tblLavA}
SQuAD & cv        & 0.3339 & 1.3639 & 0.5040 & 0.4913 \\
\rowcolor{tblLavB}
SQuAD & norm-cv   & 0.3377 & 1.3866 & 0.5065 & 0.4940 \\
\bottomrule
\end{tabular}
\caption{\textbf{Sensitivity to the token-level uncertainty definition.}
We replace the default dimension-averaged standard deviation definition of $U(x_j)$ with alternative dispersion summaries, including \texttt{cov-trace} (total variance via $\mathrm{tr}(\mathrm{Cov})$) and scale-normalized coefficient-of-variation variants (\texttt{cv}, \texttt{norm-cv}), and reevaluate MNLI and SQuAD.
Performance remains largely stable across variants in terms of ECE, NLL, accuracy, and AURC, and the overall conclusions are unchanged.
This suggests that UAT-LITE is not dependent on a single specific aggregation of embedding-level stochasticity.}
\label{tab:variants}
\end{table}

\subsection{Negation Diagnostics and Qualitative Attention Illustration}
\label{app:negation}

We include a controlled negation perturbation as a \emph{diagnostic} of uncertainty sensitivity under
semantic destabilization rather than as a correctness test.
Because negation can legitimately change the true label, label-dependent quantities such as AccNeg and
$\Delta$Err, both computed against the \emph{original} labels, are only partially informative.
We therefore emphasize \emph{label-free stability signals} that reflect whether the model expresses increased
uncertainty under the perturbation: the change in confidence ($\Delta$Conf), the change in the margin between the
top-1 and top-2 predicted probabilities ($\Delta$Margin), and the change in predictive entropy ($\Delta$Entropy),
where larger entropy and smaller margin indicate greater uncertainty.

This diagnostic tests whether uncertainty estimates respond \emph{directionally} to a controlled semantic change,
rather than whether the model is invariant to arbitrary linguistic rephrasings.
Accordingly, we do not claim that calibration behavior under negation generalizes to all paraphrases or stylistic variants.
Instead, the template below provides a standardized perturbation that enables consistent comparison across methods under the same intervention.

\begin{table}[t]
\centering
\setlength{\tabcolsep}{4.6pt}
\renewcommand{\arraystretch}{1.10}
\small

\resizebox{\columnwidth}{!}{%
\begin{tabular}{lcccccc}
\toprule
\rowcolor{tblNegHead}
\textbf{Method} & \textbf{AccOrig} & \textbf{AccNeg} &
$\boldsymbol{\Delta}$\textbf{Conf} &
$\boldsymbol{\Delta}$\textbf{Margin} &
$\boldsymbol{\Delta}$\textbf{Entropy} &
$\boldsymbol{\Delta}$\textbf{Err} \\
\midrule

\rowcolor{tblNegBase}
Base (deterministic)
& 0.705 & 0.495 & $-0.0831$ & $-0.1422$ & $+0.1524$ & $+0.2100$ \\

\rowcolor{tblNegUAT}
UAT-LITE (MC; Q-only default)
& 0.690 & 0.495 & $-0.0796$ & $-0.1316$ & $+0.1387$ & $+0.1950$ \\

\rowcolor{tblNegTS}
UAT-LITE + TS (MC-mean logits)
& 0.690 & 0.495 & $-0.0766$ & $-0.1224$ & $+0.1172$ & $+0.1950$ \\
\bottomrule
\end{tabular}%
}

\textcolor{blue}{
\caption{\textbf{Negation diagnostic on MNLI} ($N=200$, seed=42; template: ``It is not true that \{\emph{H}\}'').
$\Delta$ values are computed as (negated $-$ original).
Because negation may change the true label, AccNeg and $\Delta$Err, both measured against the original labels, should be interpreted cautiously.
We therefore also report label-free stability signals, namely the change in confidence, margin (top-1 minus top-2 probability), and predictive entropy.
Across methods, negation consistently lowers confidence and margin while increasing entropy, indicating greater predictive uncertainty under the perturbation.
The UAT-LITE variants exhibit slightly smaller magnitude shifts in these signals than the deterministic baseline, suggesting somewhat more controlled confidence responses under semantic destabilization.}
}
\label{tab:negation}
\end{table}

Across methods, the negation perturbation lowers confidence and margin while increasing entropy, indicating that
it induces greater predictive uncertainty.
Because the perturbation may change the underlying label, these results are best interpreted as an
\emph{uncertainty-responsiveness} check, that is, a test of directional sensitivity to semantic destabilization, rather than as a measure of
robust accuracy under label-preserving transformations.

To aid interpretability, Fig.~\ref{fig:uncertainty_attention} visualizes a representative example of
uncertainty-weighted attention under negation.
The goal is to illustrate how uncertainty-guided gating can induce small, structured redistributions
in internal information flow.
We do not treat attention as a faithful explanation mechanism, and this evidence is not used as quantitative evidence.
Because attention weights are normalized by softmax, absolute maps can appear visually similar;
we therefore additionally show a difference heatmap to highlight directional shifts.

\begin{figure*}[t]
    \centering
    \includegraphics[width=0.9\textwidth]{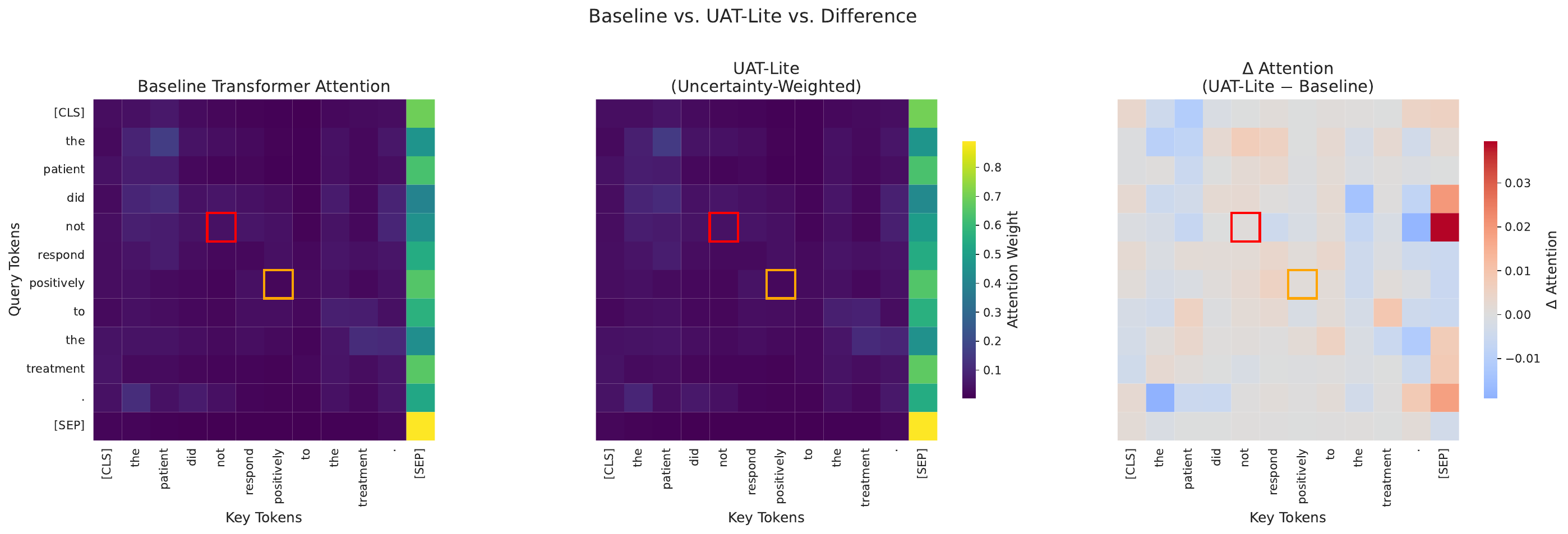}
    \caption{
    \textbf{Uncertainty-aware attention under negation (illustrative).}
    Attention heatmaps for a standard transformer (left),
    \ours (middle), and their difference (right) on the sentence
    ``The patient did not respond positively to the treatment.''
    While the absolute attention maps appear similar because of softmax normalization, the difference map highlights small but structured redistributions, with attention shifting away from sentiment-bearing tokens and toward negation-sensitive predicate structure in this example.
    This visualization is qualitative and shown for a representative run
    (single seed); it is intended to support interpretability rather than
    provide a corpus-level statistic.
    }
    \label{fig:uncertainty_attention}
\end{figure*}


\subsection{Behavioral analysis by linguistic bucket (MedQA)}
\label{app:behavior_medqa}

To examine \emph{when} uncertainty-weighted attention affects reliability, we stratify MedQA items into
coarse linguistic buckets using simple cues, including negation markers, numeric content, biomedical terms, and hedging markers.
Table~\ref{tab:behavior_medqa} reports calibration and predictive performance within each bucket.

Two patterns are consistent across buckets.
First, \textbf{UAT-LITE alone tends to improve accuracy} relative to Base+TS
(e.g., 0.2476$\rightarrow$0.2614 on negation, 0.2418$\rightarrow$0.2514 on numeric, 0.2369$\rightarrow$0.2578 on biomedical, and 0.2344$\rightarrow$0.2583 on hedges),
but it does \textbf{not} improve in-domain calibration relative to TS, which remains the stronger post-hoc calibrator.
Second, \textbf{UAT-LITE+TS yields the lowest ECE in every bucket}
(0.0110 vs.\ 0.0201 on negation; 0.0010 vs.\ 0.0257 on numeric; 0.0075 vs.\ 0.0302 on biomedical; 0.0080 vs.\ 0.0331 on hedges),
while preserving the accuracy gains introduced by UAT-LITE.
This pattern supports a complementary interpretation:
uncertainty-weighted attention improves internal evidence aggregation, whereas TS further corrects output-level confidence.

Although coverage at the 0.9 threshold is zero for all methods in these buckets, ECE, AURC, and accuracy still reveal meaningful behavioral differences.
Overall, this analysis suggests that the benefit of UAT-LITE is not confined to a single linguistic phenomenon, but appears across several medically relevant sources of ambiguity and variation.

\begin{table}[t]
\centering
\small
\setlength{\tabcolsep}{4pt}

\definecolor{tblBehHeader}{RGB}{233,199,188}
\definecolor{tblBehBucket}{RGB}{245,226,220}
\definecolor{tblBehA}{RGB}{252,246,244}
\definecolor{tblBehB}{RGB}{248,238,234}

\begin{tabular}{l r l c c c c}
\toprule
\rowcolor{tblBehHeader}
Bucket & N & Method & ECE $\downarrow$ & Cov@0.9 $\uparrow$ & AURC $\downarrow$ & Acc $\uparrow$ \\
\midrule

\rowcolor{tblBehBucket}
\textit{all} & 1273 & Base+TS & 0.0248 & 0.0000 & 0.7261 & 0.2427 \\
\rowcolor{tblBehA}
&  & UAT-LITE & 0.0809 & 0.0000 & 0.7478 & 0.2498 \\
\rowcolor{tblBehB}
&  & UAT-LITE+TS & 0.0006 & 0.0000 & 0.7493 & 0.2498 \\
\cmidrule{1-7}

\rowcolor{tblBehBucket}
\textit{negation} & 723 & Base+TS & 0.0201 & 0.0000 & 0.7213 & 0.2476 \\
\rowcolor{tblBehA}
&  & UAT-LITE & 0.0677 & 0.0000 & 0.7358 & 0.2614 \\
\rowcolor{tblBehB}
&  & UAT-LITE+TS & 0.0110 & 0.0000 & 0.7362 & 0.2614 \\
\cmidrule{1-7}

\rowcolor{tblBehBucket}
\textit{numeric} & 1245 & Base+TS & 0.0257 & 0.0000 & 0.7231 & 0.2418 \\
\rowcolor{tblBehA}
&  & UAT-LITE & 0.0793 & 0.0000 & 0.7441 & 0.2514 \\
\rowcolor{tblBehB}
&  & UAT-LITE+TS & 0.0010 & 0.0000 & 0.7455 & 0.2514 \\
\cmidrule{1-7}

\rowcolor{tblBehBucket}
\textit{biomedical} & 574 & Base+TS & 0.0302 & 0.0000 & 0.7259 & 0.2369 \\
\rowcolor{tblBehA}
&  & UAT-LITE & 0.0716 & 0.0000 & 0.7371 & 0.2578 \\
\rowcolor{tblBehB}
&  & UAT-LITE+TS & 0.0075 & 0.0000 & 0.7357 & 0.2578 \\
\cmidrule{1-7}

\rowcolor{tblBehBucket}
\textit{hedges} & 751 & Base+TS & 0.0331 & 0.0000 & 0.7471 & 0.2344 \\
\rowcolor{tblBehA}
&  & UAT-LITE & 0.0702 & 0.0000 & 0.7351 & 0.2583 \\
\rowcolor{tblBehB}
&  & UAT-LITE+TS & 0.0080 & 0.0000 & 0.7362 & 0.2583 \\
\bottomrule
\end{tabular}

\caption{\textbf{Behavioral analysis on MedQA by linguistic bucket.}
Items are stratified using lightweight heuristic cues for negation, numeric content, biomedical terminology, and hedging.
Across all buckets, UAT-LITE improves accuracy relative to Base+TS, whereas UAT-LITE+TS achieves the lowest ECE in every case.
The largest calibration gains from stacking with TS appear in the numeric and hedge buckets, while the strongest accuracy gains from UAT-LITE appear in the biomedical and hedge buckets.
Although Cov@0.9 is zero throughout, ECE, AURC, and accuracy still show consistent differences across methods, supporting the view that uncertainty-aware attention and TS play complementary roles in this setting.}
\label{tab:behavior_medqa}
\end{table}

\subsection{Sensitivity Analysis of $\lambda$ and $M$}
\label{app:sensitivity}

We examine the sensitivity of \ours to two inference-time hyperparameters:
the uncertainty weighting coefficient $\lambda$ and the Monte Carlo budget $M$.
We evaluate a representative grid with $\lambda \in \{0.1, 0.5, 1.0\}$ and
$M \in \{3, 5, 10\}$ on SQuAD~2.0, which is particularly sensitive to epistemic
uncertainty due to the presence of both answerable and unanswerable questions.
Similar stability trends were observed on MNLI and SST-2 and are omitted for brevity.
We report ECE and accuracy to assess calibration quality and task performance; accuracy is included
to verify that calibration changes are not driven by degraded performance.

Table~\ref{tab:sensitivity} shows that calibration remains stable across the grid:
ECE lies in a narrow interval $[0.0640,\,0.0703]$ (mean $=0.0670$, std $=0.0022$).
No sharp sensitivity to $\lambda$ is observed; across $\lambda$ blocks, ECE varies by at most $\approx 0.006$.
Within each $\lambda$, increasing $M$ tends to reduce ECE relative to the $M{=}3$ reference
($\Delta$ECE is consistently non-positive), but gains diminish beyond $M{=}5$:
moving from $M{=}3 \rightarrow 5$ yields most of the improvement at a moderate cost increase
(MC Cost $\approx 1.67\times$), while $M{=}10$ provides smaller additional benefits despite
a $3.33\times$ cost. Accuracy is similarly robust: changes within each $\lambda$ block are small
($\Delta$Acc $\le 0.0037$ absolute), indicating that uncertainty-weighted attention primarily
modulates calibration rather than task accuracy.
Overall, the method is insensitive to precise tuning of $\lambda$ and achieves stable calibration
over a broad range of inference-time settings, supporting practicality when extensive
hyperparameter search is undesirable.


\definecolor{sensHead}{RGB}{235,235,235}   
\definecolor{sensSep}{RGB}{245,245,245}    
\definecolor{sensLamA}{RGB}{255,247,230}   
\definecolor{sensLamB}{RGB}{233,245,255}   
\definecolor{sensLamC}{RGB}{246,236,255}   

\begin{table}[t]
\centering
\small
\setlength{\tabcolsep}{4.5pt}
\renewcommand{\arraystretch}{1.10}

\begin{adjustbox}{max width=\columnwidth,center}
\begin{tabular}{cc cc cc c}
\toprule
\rowcolor{sensHead}
$\boldsymbol{\lambda}$ & $\boldsymbol{M}$ &
\textbf{ECE$\downarrow$} & \textbf{Acc$\uparrow$} &
$\boldsymbol{\Delta}$\textbf{ECE} & $\boldsymbol{\Delta}$\textbf{Acc} &
\textbf{MC Cost$\times$} \\
\rowcolor{sensHead}
& &
& &
\footnotesize{(vs.\ $M{=}3$)} & \footnotesize{(vs.\ $M{=}3$)} &
\footnotesize{(vs.\ $M{=}3$)} \\
\midrule

\rowcolor{sensSep}
\multicolumn{7}{l}{\textbf{$\lambda=0.1$}} \\
\rowcolor{sensLamA}
0.1 & 3  & 0.0699 & 0.7075 & 0.0000 & 0.0000 & 1.00 \\
\rowcolor{sensLamA}
0.1 & 5  & 0.0676 & 0.7075 & $-0.0023$ & $+0.0000$ & 1.67 \\
\rowcolor{sensLamA}
0.1 & 10 & 0.0649 & 0.7109 & $-0.0050$ & $+0.0034$ & 3.33 \\
\midrule

\rowcolor{sensSep}
\multicolumn{7}{l}{\textbf{$\lambda=0.5$}} \\
\rowcolor{sensLamB}
0.5 & 3  & 0.0703 & 0.7059 & 0.0000 & 0.0000 & 1.00 \\
\rowcolor{sensLamB}
0.5 & 5  & 0.0657 & 0.7096 & $-0.0046$ & $+0.0037$ & 1.67 \\
\rowcolor{sensLamB}
0.5 & 10 & 0.0682 & 0.7068 & $-0.0021$ & $+0.0009$ & 3.33 \\
\midrule

\rowcolor{sensSep}
\multicolumn{7}{l}{\textbf{$\lambda=1.0$}} \\
\rowcolor{sensLamC}
1.0 & 3  & 0.0669 & 0.7090 & 0.0000 & 0.0000 & 1.00 \\
\rowcolor{sensLamC}
1.0 & 5  & 0.0655 & 0.7102 & $-0.0014$ & $+0.0012$ & 1.67 \\
\rowcolor{sensLamC}
1.0 & 10 & \textbf{0.0640} & \textbf{0.7110} & $-0.0029$ & $+0.0020$ & 3.33 \\
\bottomrule
\end{tabular}
\end{adjustbox}

\vspace{0.6em}
\begin{minipage}{0.96\columnwidth}
\footnotesize
\textbf{Summary.} ECE mean $=0.0670$, std $=0.0022$; min $=0.0640$, max $=0.0703$ (range $=0.0063$).
$\Delta$ECE/$\Delta$Acc are computed within each $\lambda$ block relative to $M{=}3$; MC Cost scales linearly with $M$.
\end{minipage}

\caption{Sensitivity analysis of \ours with respect to uncertainty penalty $\lambda$ and Monte Carlo budget $M$ on SQuAD~2.0 (single seed).
Calibration remains stable across a wide range of inference-time settings (ECE range $<0.007$).
Increasing $M$ yields diminishing returns beyond $M=5$, while accuracy changes remain small, indicating limited sensitivity to precise hyperparameter tuning.}
\label{tab:sensitivity}
\vspace{-2mm}
\end{table}


\subsection{Character-Level Adversarial Robustness}
\label{app:appendix_adversarial}

Table~\ref{tab:adversarial} evaluates robustness under character-level perturbations
on SST-2.
We apply random 5\% character swaps to simulate typographical noise and simple adversarial corruption.
Both the baseline BERT model and \ours exhibit substantial accuracy degradation
under attack, with accuracy dropping from 0.920 to 0.816 for the baseline and from 0.922 to 0.812 for \ours.
Calibration also deteriorates markedly, with ECE increasing
from 0.062 to 0.133 for the baseline and from 0.055 to 0.132 for \ours.

Although \ours achieves slightly better clean accuracy (0.922 vs.\ 0.920)
and clean calibration (ECE 0.055 vs.\ 0.062), its robustness under attack remains broadly comparable to that of the baseline.
These results suggest that uncertainty-aware attention improves calibration
under clean conditions but provides limited additional robustness to character-level perturbations.
This pattern is consistent with prior work showing that calibration-oriented methods primarily address
uncertainty estimation and do not, by themselves, constitute adversarial defenses
\cite{goodfellow2014explaining}.
An important direction for future work is to combine \ours with dedicated adversarial-defense strategies,
such as adversarial training or certified robustness methods.

\begin{table*}[t]
\centering
\small
\renewcommand{\arraystretch}{1.15}
\setlength{\tabcolsep}{6pt}

\definecolor{tblAdvHeader}{RGB}{224,179,179}
\definecolor{tblAdvA}{RGB}{250,241,241}
\definecolor{tblAdvB}{RGB}{245,231,231}

\begin{tabular}{l l c c c c}
\toprule
\rowcolor{tblAdvHeader}
\textbf{Attack} & \textbf{Method} & \textbf{Clean Acc} & \textbf{Attack Acc} & \textbf{Clean ECE} & \textbf{Attack ECE} \\
\midrule
\rowcolor{tblAdvA}
Character Swap & Baseline & 0.920 & 0.816 & 0.062 & 0.133 \\
\rowcolor{tblAdvB}
Character Swap & \ours & 0.922 & 0.812 & 0.055 & 0.132 \\
\bottomrule
\end{tabular}

\caption{\textbf{Character-level adversarial robustness on SST-2} (single seed).
We evaluate 500 validation examples under 5\% random character swaps using a fine-tuned BERT-base model.
Both the baseline and \ours suffer substantial degradation in accuracy and calibration under perturbation.
Although \ours shows slightly better clean accuracy and clean ECE, attack-time accuracy and ECE are nearly unchanged relative to the baseline.
These results indicate that the calibration benefits of uncertainty-aware attention on clean data do not directly translate into stronger robustness against character-level perturbations.}
\label{tab:adversarial}
\end{table*}

\subsection{Full Linguistic Uncertainty Analysis}
\label{app:full_linguistic_uncertainty}

Table~\ref{tab:linguistic} provides a fine-grained analysis of how
\ours responds to diverse sources of uncertainty spanning
lexical, syntactic, epistemic, and logical phenomena.
For \textbf{negation}, \ours remains uniformly confident across affirmative,
negated, and clinical-negation cases, yielding near-zero estimated uncertainty.
This behavior indicates invariance to explicit negation cues, which is desirable
for clear clinical negation (e.g., ``not stable'') but also suggests limited
sensitivity to polarity reversal in more general contexts.
This pattern is consistent with the attention analysis in
\cref{fig:uncertainty_attention}, which shows that uncertainty-aware
attention redistributes weight conservatively toward negation and predicate
structure without producing elevated output-level uncertainty for clear
clinical negation.

For \textbf{lexical ambiguity}, the behavior is mixed.
When contextual disambiguation is available (``bank account''),
\ours shows a meaningful increase in uncertainty.
However, for genuinely ambiguous expressions (``bank closed''), the model remains highly
confident, indicating residual overconfidence under unresolved polysemy.
A similar intermediate pattern appears for polysemous triggers such as ``bat,''
where uncertainty increases but does not fully reflect the underlying ambiguity.

A related pattern appears for \textbf{epistemic hedges}.
Strong assertions (``definitely'') and weaker hedges (``probably,'' ``possibly'')
are treated with near-maximal confidence, indicating limited sensitivity to
graded epistemic marking.
By contrast, explicit modal uncertainty (``might'') leads to a measurable
increase in uncertainty, suggesting that the model responds more strongly to
direct lexical uncertainty markers than to subtler hedging cues.

For \textbf{syntactic complexity}, \ours shows clearer sensitivity.
Medium-depth structures produce the largest increase in uncertainty, consistent
with greater structural processing difficulty.
Deeper structures, however, elicit a non-monotonic response, suggesting that the
model may partially adapt to highly regular but more complex constructions
rather than increasing uncertainty monotonically with depth.
Finally, under \textbf{logical contradictions}, \ours appropriately increases
uncertainty for inconsistent inputs (``healthy $\wedge$ dying'') while maintaining high confidence for coherent statements.
This pattern indicates sensitivity to semantic inconsistency and is consistent
with the framework’s uncertainty propagation across layers.


\definecolor{lingHead}{RGB}{242,242,242}   
\definecolor{lingSection}{RGB}{242,242,242} 

\definecolor{lingGood}{RGB}{255,241,230}   
\definecolor{lingWarn}{RGB}{255,232,214}   
\definecolor{lingBad}{RGB}{255,220,220}    

\begin{table*}[t]
\centering
\scriptsize
\setlength{\tabcolsep}{6pt}
\renewcommand{\arraystretch}{1.15}

\begin{tabular}{l l c c c l}
\toprule
\rowcolor{lingHead}
\textbf{Phenomenon} &
\textbf{Trigger} &
\textbf{Base Conf} &
\textbf{MB Conf} &
\textbf{Unc} &
\textbf{Diagnostic} \\
\midrule

\rowcolor{lingSection}
\multicolumn{6}{l}{\textbf{Negation}} \\
\rowcolor{lingGood}
Positive      & excellent             & 0.53 & 1.00 & \textit{0.00} & Negation-invariant \\
\rowcolor{lingGood}
Negated       & not excellent         & 0.52 & 1.00 & \textit{0.00} & Negation-invariant \\
\rowcolor{lingGood}
Clinical      & not stable            & 0.55 & 1.00 & \textit{0.00} & Confident clinical negation \\

\midrule
\rowcolor{lingSection}
\multicolumn{6}{l}{\textbf{Lexical Ambiguity}} \\
\rowcolor{lingWarn}
Clear context & bank account          & 0.51 & 0.70 & \textbf{0.30} & Appropriate ambiguity \\
\rowcolor{lingBad}
Ambiguous     & bank closed           & 0.55 & 0.99 & \textit{0.01} & Overconfident ambiguity \\
\rowcolor{lingWarn}
Polysemous    & bat (animal/tool)     & 0.54 & 0.83 & 0.17          & Partial polysemy sensitivity \\

\midrule
\rowcolor{lingSection}
\multicolumn{6}{l}{\textbf{Epistemic Hedges}} \\
\rowcolor{lingGood}
Definite      & definitely            & 0.51 & 1.00 & \textit{0.00} & Strong assertion \\
\rowcolor{lingBad}
Probable      & probably              & 0.53 & 0.99 & \textit{0.01} & Hedge-insensitive \\
\rowcolor{lingBad}
Possible      & possibly              & 0.52 & 0.99 & \textit{0.01} & Hedge-insensitive \\
\rowcolor{lingWarn}
Uncertain     & might                 & 0.51 & 0.85 & 0.15          & Uncertainty-aware \\

\midrule
\rowcolor{lingSection}
\multicolumn{6}{l}{\textbf{Syntactic Complexity}} \\
\rowcolor{lingGood}
Simple        & depth=2               & 0.55 & 0.82 & 0.18          & Low structural load \\
\rowcolor{lingWarn}
Medium        & depth=5               & 0.53 & 0.54 & \textbf{0.46} & Complexity-sensitive \\
\rowcolor{lingWarn}
Complex       & depth=8               & 0.53 & 0.83 & 0.17          & Non-monotonic response \\

\midrule
\rowcolor{lingSection}
\multicolumn{6}{l}{\textbf{Contradictions}} \\
\rowcolor{lingGood}
Consistent    & healthy, stable        & 0.52 & 0.99 & \textit{0.01} & Coherent input \\
\rowcolor{lingWarn}
Contradictory & healthy $\wedge$ dying & 0.53 & 0.90 & 0.10          & Contradiction-aware \\

\bottomrule
\end{tabular}

\caption{\textbf{Fine-grained linguistic uncertainty analysis across diverse phenomena} (single seed).
The table reports baseline confidence, \ours confidence (MB), derived epistemic uncertainty, and a qualitative diagnostic interpretation for representative triggers spanning negation, lexical ambiguity, epistemic hedging, syntactic complexity, and contradiction.
Several patterns are notable.
Negation and explicit clinical negation produce near-zero uncertainty, indicating stable behavior in these cases.
By contrast, unresolved lexical ambiguity and weaker hedges often remain associated with high confidence, revealing residual overconfidence.
The strongest uncertainty increases appear for medium syntactic complexity and logical contradiction, suggesting greater sensitivity to structural difficulty and semantic inconsistency than to all forms of lexical or epistemic ambiguity.
Overall, the table highlights both the strengths and the limitations of uncertainty-weighted attention at a fine-grained linguistic level.}
\label{tab:linguistic}
\vspace{-2mm}
\end{table*}

\section{Cross-Model \& Scaling Analysis}
\label{app:appendix_cross_model}

\textbf{Scope clarification.}
All models evaluated in this section are encoder-only,
BERT-derived transformer architectures.
Generalization to fundamentally different architectures, such as decoder-only
models (e.g., GPT or LLaMA) or state-space models, is outside the scope of the
current study and is left for future work.

\subsection{Generalization Across BERT Backbones}
\label{app:appendix_cross_domain}

Table~\ref{tab:cross_model_summary} evaluates whether the calibration benefits
of \ours generalize across \emph{BERT-based} transformer backbones pretrained
on heterogeneous corpora and applied to domain-specific benchmarks.
The evaluated models span general-purpose (BERT-base, RoBERTa-base),
clinical (ClinicalBERT), biomedical (BioBERT), and scientific (SciBERT)
representations, thereby testing robustness across substantially different
linguistic and semantic domains.

Across the evaluated settings, \ours reduces ECE for most pretrained
backbones, yielding an average relative improvement of approximately \textbf{32\%}
(Table~\ref{tab:cross_model_summary}).
The largest gain is observed for BERT-base on SQuAD, where
ECE decreases by \textbf{52\%} (0.130 $\rightarrow$ 0.062).
Because SQuAD contains both answerable and unanswerable questions, it places
substantial demands on ambiguity recognition and confidence estimation.
The large improvement in this setting suggests that uncertainty-weighted
attention helps the model better calibrate predictions when evidence
aggregation is uncertain.

Domain-specialized models also benefit, although to different degrees.
BioBERT on PubMedQA shows a \textbf{40\%} reduction in ECE, suggesting that
uncertainty-aware attention is useful for evidence-dense biomedical text.
ClinicalBERT on MedQA shows a \textbf{25\%} improvement, indicating that the
method remains effective in fact-oriented medical question answering, although
the gain is smaller than in some other settings.
By contrast, RoBERTa-base on MNLI shows only a marginal \textbf{2\%}
improvement.
MNLI is a high-resource benchmark with strong pretrained representations and
a relatively well-calibrated baseline, leaving less room for further gains.
SciBERT on SQuAD similarly shows a modest \textbf{10\%} reduction, suggesting
that pretraining on specialized scientific corpora may already capture some
task-relevant uncertainty structure.

Taken together, these results indicate that the calibration benefits of \ours
transfer across multiple BERT-derived backbones and application domains, with
the largest gains arising in settings characterized by greater ambiguity or
less favorable baseline calibration.


\definecolor{cmHead}{RGB}{242,242,242}      
\definecolor{cmSection}{RGB}{242,242,242}   

\definecolor{cmRowA}{RGB}{236,248,247}      
\definecolor{cmRowB}{RGB}{246,252,252}      
\definecolor{cmAvg}{RGB}{223,242,239}       
\definecolor{cmImp}{RGB}{230,245,255}       

\begin{table*}[t]
\centering
\small
\setlength{\tabcolsep}{6pt}
\renewcommand{\arraystretch}{1.15}

\begin{tabular}{l l l c c c}
\toprule
\rowcolor{cmHead}
\textbf{Base Model} & \textbf{Domain} & \textbf{Dataset} &
\textbf{Baseline ECE} & \textbf{+\ours ECE} & \textbf{Improvement} \\
\midrule

\rowcolor{cmRowA}
BERT-base      & General      & SQuAD     & 0.130 & 0.062 & \cellcolor{cmImp}52\% $\downarrow$ \\
\rowcolor{cmRowB}
RoBERTa-base   & General      & MNLI      & 0.040 & 0.039 & \cellcolor{cmImp}2\% $\downarrow$ \\
\rowcolor{cmRowA}
ClinicalBERT   & Medical      & MedQA     & 0.049 & 0.037 & \cellcolor{cmImp}25\% $\downarrow$ \\
\rowcolor{cmRowB}
BioBERT        & Biomedical   & PubMedQA  & 0.120 & 0.072 & \cellcolor{cmImp}40\% $\downarrow$ \\
\rowcolor{cmRowA}
SciBERT        & Scientific   & SQuAD     & 0.092 & 0.083 & \cellcolor{cmImp}10\% $\downarrow$ \\

\midrule
\rowcolor{cmAvg}
\textbf{Average} & — & — & \textbf{0.086} & \textbf{0.059} & \cellcolor{cmImp}\textbf{32\% $\downarrow$} \\
\bottomrule
\end{tabular}

\caption{\textbf{Cross-domain calibration across pretrained BERT-derived backbones} (single seed).
Expected Calibration Error (ECE; lower is better) is reported before and after applying \ours, with relative improvement computed against each model’s own baseline.
The evaluated models span general-purpose, clinical, biomedical, and scientific pretraining regimes, allowing comparison across substantially different domains and benchmark types.
Improvements are observed for most backbones, with the largest gains on BERT-base for SQuAD and BioBERT for PubMedQA, while RoBERTa-base on MNLI shows only a marginal change.
Overall, the table suggests that the calibration benefits of \ours are not confined to a single pretrained model or domain, although the magnitude of improvement varies with the benchmark and the strength of the starting baseline.}
\label{tab:cross_model_summary}
\vspace{-2mm}
\end{table*}

\subsection{Effect of Model Capacity on Calibration}
\label{app:appendix_bert_scaling}

Table~\ref{tab:bert_scaling} analyzes the effect of model capacity on calibration
improvement by applying \ours across BERT models ranging from BERT-tiny (4.4M parameters)
to BERT-large (335.1M parameters).
Unlike Table~\ref{tab:cross_model_summary}, which varies domains and pretraining corpora,
this analysis holds the architecture family fixed and isolates the role of
parameter scale.

Small and mid-sized models benefit most from uncertainty-weighted attention.
BERT-mini achieves the largest relative improvement, with a
\textbf{58\%} reduction in ECE, followed by BERT-base with an approximately
\textbf{53\%} reduction.
These models appear to have sufficient representational capacity to express
epistemic uncertainty while still benefiting from explicit uncertainty
modulation within attention.

BERT-tiny and BERT-small show more modest improvements (12\% and 15\%,
respectively).
For very small models, limited representational capacity may constrain the
extent to which uncertainty-aware inference can reshape internal attention
patterns, yielding smaller but still consistent gains.
In contrast, BERT-large shows degraded calibration, with ECE increasing by
\textbf{13\%} after applying \ours.
One possible explanation is that, in highly overparameterized models,
stochastic perturbations introduced through Monte Carlo sampling may disrupt
already stable attention patterns, leading to excess variance propagation
rather than improved uncertainty estimation.

Latency measurements further contextualize these results.
Although inference-time overhead remains moderate for mid-sized models
(e.g., 7 ms $\rightarrow$ 39 ms for BERT-base), it increases substantially for larger
architectures (26 ms $\rightarrow$ 134 ms for BERT-large).
Overall, this analysis suggests that uncertainty-weighted attention is most
effective in the small-to-mid-sized regime, where model capacity and stochasticity
appear to be better balanced.


\definecolor{scHead}{RGB}{242,242,242}   
\definecolor{scRowA}{RGB}{246,242,252}   
\definecolor{scRowB}{RGB}{252,250,255}   
\definecolor{scWarn}{RGB}{255,238,238}   
\definecolor{scDelta}{RGB}{237,246,255}  
\definecolor{scLat}{RGB}{245,255,245}    

\begin{table*}[t]
\centering
\small
\setlength{\tabcolsep}{5pt}
\renewcommand{\arraystretch}{1.15}

\begin{tabular}{l c c c c c}
\toprule
\rowcolor{scHead}
\textbf{Model} & \textbf{Params} &
\textbf{Baseline ECE} & \textbf{\ours ECE} &
\textbf{Rel. $\Delta$} & \textbf{Latency} \\
\midrule

\rowcolor{scRowA}
BERT-tiny  & 4.4M   & 0.070 & 0.061 & \cellcolor{scDelta}12\% $\downarrow$ & \cellcolor{scLat}1 $\rightarrow$ 5 \\
\rowcolor{scRowB}
BERT-mini  & 11.2M  & 0.131 & 0.056 & \cellcolor{scDelta}58\% $\downarrow$ & \cellcolor{scLat}1 $\rightarrow$ 8 \\
\rowcolor{scRowA}
BERT-small & 28.8M  & 0.095 & 0.081 & \cellcolor{scDelta}15\% $\downarrow$ & \cellcolor{scLat}1 $\rightarrow$ 9 \\
\rowcolor{scRowB}
BERT-base  & 109.5M & 0.139 & 0.066 & \cellcolor{scDelta}53\% $\downarrow$ & \cellcolor{scLat}7 $\rightarrow$ 39 \\

\rowcolor{scWarn}
BERT-large & 335.1M & 0.152 & 0.172 & \cellcolor{scDelta}13\% $\uparrow$ & \cellcolor{scLat}26 $\rightarrow$ 134 \\

\bottomrule
\end{tabular}

\caption{\textbf{Effect of model capacity on calibration within the BERT family} (single seed).
Expected Calibration Error (ECE; lower is better) is reported before and after applying \ours across models ranging from BERT-tiny to BERT-large.
Relative change is computed with respect to each model’s baseline ECE, and latency reports baseline-to-\ours inference time in milliseconds.
The largest improvements occur for BERT-mini and BERT-base, whereas BERT-tiny and BERT-small show smaller but still positive gains.
In contrast, BERT-large exhibits a degradation in calibration together with the highest inference overhead.
These results suggest that the effectiveness of uncertainty-weighted attention depends in part on model scale, with the most favorable trade-off appearing in the small-to-mid-sized regime.}
\label{tab:bert_scaling}
\vspace{-2mm}
\end{table*}


\section{Implementation and Reproducibility}
\label{app:complexity-analysis}

This section characterizes the computational and implementation properties of
\ours.
We first analyze its asymptotic and practical computational complexity relative to
standard transformer inference and ensemble-based uncertainty methods.
We then quantify the inference-time overhead introduced by Monte Carlo sampling and
summarize the experimental setup, hardware configuration, and reproducibility details
required to replicate the reported results.
Throughout, we emphasize that \ours operates strictly at
\emph{inference time} and does not alter pretrained parameters or training cost.

\subsection{Computational Complexity Analysis}
\label{subsec:complexity}

\ours preserves the parameter count and training cost of the base transformer.
All uncertainty-aware mechanisms operate strictly at inference time and introduce
no additional trainable parameters or architectural modifications.

The additional cost arises solely from Monte Carlo (MC) sampling during inference:
\begin{itemize}
    \item \textbf{Parameters:} identical to the base model (e.g., 109M for BERT-base)
    \item \textbf{Training cost:} unchanged
    \item \textbf{Calibration cost:} one scalar temperature per task (optional)
    \item \textbf{Inference cost:} $M$ stochastic forward passes ($M\in\{5,10\}$)
\end{itemize}

Inference-time complexity therefore scales linearly with the MC budget $M$.
Table~\ref{tab:efficiency} reports a canonical end-to-end GPU benchmark at $M{=}10$
under a sequential MC evaluation loop.
Relative to deterministic inference, \ours incurs approximately a $23\times$
slowdown, whereas temperature scaling adds negligible overhead because it applies
only a scalar rescaling to the logits.
Unlike Deep Ensembles, which require training, storing, and evaluating multiple
independent models, \ours introduces no additional training or model-storage cost,
making it a lighter-weight alternative when retraining is undesirable.


\definecolor{effHead}{RGB}{242,242,242}    
\definecolor{effRowA}{RGB}{248,249,250}    
\definecolor{effRowB}{RGB}{255,255,255}    
\definecolor{effHeavy}{RGB}{255,236,236}   
\definecolor{effRatio}{RGB}{232,243,255}   

\begin{table}[t]
\centering
\setlength{\tabcolsep}{5pt}
\renewcommand{\arraystretch}{1.12}
\footnotesize

\begin{tabular}{lccc}
\toprule
\rowcolor{effHead}
\textbf{Method} & \textbf{Latency (ms)} & $\pm$ & \textbf{Ratio} \\
\midrule

\rowcolor{effRowA}
Base (det.)        & 62.93   & 0.29  & \cellcolor{effRatio}1.00$\times$ \\

\rowcolor{effRowB}
Base + TS          & 63.13   & 0.27  & \cellcolor{effRatio}1.00$\times$ \\

\rowcolor{effHeavy}
\ours\ (M=10)      & 1426.90 & 78.87 & \cellcolor{effRatio}22.68$\times$ \\

\rowcolor{effHeavy}
\ours\ + TS (M=10) & 1486.61 & 7.26  & \cellcolor{effRatio}23.62$\times$ \\

\bottomrule
\end{tabular}

\caption{\textbf{Canonical GPU inference latency at $M{=}10$.}
Latency is measured on a single NVIDIA A100 GPU using batch size $32$, sequence length $128$, and a sequential Monte Carlo evaluation loop.
Values report end-to-end forward-pass wall-clock time per batch (mean $\pm$ standard deviation), together with the ratio relative to deterministic inference under the same setting.
As expected, \ours and \ours+TS incur substantial overhead because they require repeated stochastic forward passes, whereas Base+TS is effectively identical to the deterministic baseline.
The small difference between \ours and \ours+TS confirms that the additional cost is dominated by MC sampling rather than by temperature scaling itself.}
\label{tab:efficiency}
\vspace{-4pt}
\end{table}

\subsection{Inference-Time Overhead}
\label{app:inference-overhead}

While the previous subsection characterizes asymptotic complexity, we now report
empirical inference-time overhead under realistic deployment conditions.
At test time, \ours performs $M$ stochastic forward passes through
the same pretrained model.
When executed sequentially, runtime increases approximately linearly with $M$.
Observed slowdowns depend on the Monte Carlo budget $M$ and the latency measurement
protocol, including batching, sequence length, and whether MC passes are executed
sequentially or vectorized.
Throughout this work, we report latency under the canonical benchmark
described in \ref{app:hardware}, under which overhead scales approximately linearly
with $M$ for sequential MC inference.

Despite this overhead, inference-time cost remains practical in settings where
reliability and calibrated confidence are prioritized over raw throughput.
Importantly, this cost is incurred only when uncertainty-aware inference is enabled and
can in principle be reduced by lowering $M$ in settings where a smaller stochastic budget is acceptable, as discussed in Section~\ref{sec:discussion}.

\subsection{Reproducible Experimental Setup}
\label{app:appendix_repro}

This subsection summarizes the experimental setup and implementation details necessary
to reproduce our results.
All experiments follow a consistent architecture, training configuration, and
evaluation protocol across datasets.

\paragraph{Model architecture.}
\begin{itemize}
    \item Base model: \texttt{BERT-base-uncased}
    \item Number of parameters: 110M
    \item Transformer layers: 12
    \item Hidden dimension: 768
    \item Attention heads: 12
\end{itemize}

\paragraph{Training hyperparameters.}
\begin{itemize}
    \item Optimizer: AdamW with standard settings
    \item Learning rate: $2 \times 10^{-5}$
    \item Learning rate schedule: linear warmup for the first 500 steps, followed by linear decay
    \item Batch size: 16
    \item Number of epochs: 3
    \item Weight decay: 0.01
    \item Gradient clipping: 1.0
\end{itemize}

\paragraph{\ours configuration.}
\begin{itemize}
    \item Inference via Monte Carlo sampling with $M \in \{5,10\}$ stochastic forward passes
    (default $M{=}10$ for the main revision tables; $M{=}5$ is additionally reported where noted)
    \item Component-specific dropout rates:
    \begin{itemize}
        \item Embedding layer: 0.1
        \item Attention layers: 0.2
        \item Feed-forward layers: 0.3
    \end{itemize}
    \item Uncertainty weighting coefficient: $\lambda = 0.5$ (set once on development data and held fixed across tasks)
    \item Sensitivity analysis for $\lambda$ and $M$ reported in \ref{app:sensitivity}
\end{itemize}

\subsection{Hardware, Latency, and Code Availability}
\label{app:hardware}

\paragraph{Compute resources and latency measurement.}
All experiments are conducted on a single NVIDIA A100 GPU (40GB).
Unless otherwise stated, latency is measured as \emph{end-to-end forward-pass
wall-clock time}, averaged over 200 runs after 50 warm-up iterations.
We report a single canonical latency benchmark throughout the paper:
batch size $32$, sequence length $128$, a sequential MC loop without vectorization across
MC samples, and a single GPU.
Latency is reported per batch and expressed as a ratio relative to deterministic
inference under the same conditions.
For stochastic methods, latency is measured using the same Monte Carlo budget $M$
used during evaluation.
Reported slowdowns therefore reflect the ratio between stochastic and deterministic
inference under identical measurement settings.

\paragraph{Code availability.}
To support reproducibility during anonymous review, we provide an anonymized repository containing the complete implementation, including training scripts, evaluation pipelines, and configuration files.
All experiments are implemented using HuggingFace Transformers, PyTorch, and Python.
The repository is available at:
\url{https://github.com/eliashossain001/uq_decomposition/tree/main}.


\section{Supplementary Related Work}
\label{app:related}

This section provides an extended discussion of prior work on calibration,
uncertainty estimation, and uncertainty-aware Transformer modeling.
Because the main paper focuses on method design and empirical evaluation rather
than a comprehensive survey, we defer detailed comparisons to this supplementary
section.
We organize related work into three themes:
(i) post-hoc calibration and ensemble-based uncertainty estimation,
(ii) Bayesian Transformers and stochastic inference methods, and
(iii) approaches that explicitly incorporate uncertainty into attention mechanisms.
This structure clarifies how \ours differs from and complements
existing techniques.

\subsection{Calibration and Uncertainty}

Despite high predictive accuracy, modern neural networks are often poorly calibrated
\citep{guo2017calibration}. Post-hoc calibration methods such as temperature scaling
\citep{guo2017calibration}, Platt scaling \citep{platt1999probabilistic}, and histogram
binning \citep{zadrozny2001obtaining} adjust output confidence without altering internal
representations. While simple and effective, these methods are limited by the quality
of learned features and do not affect model reasoning.

Ensemble-based approaches improve calibration by averaging predictions from multiple
models. Deep Ensembles \citep{lakshminarayanan2017simple} and snapshot ensembles
\citep{huang2017snapshot} are among the strongest calibration baselines, but require
training, storing, and evaluating multiple model replicas, resulting in substantial
computational and memory overhead. More efficient approximations, such as ensembling
pruned attention heads within a single Transformer
\citep{gabetni2025ensemblingprunedattentionheads}, reduce cost but still rely on
architectural modification or fused ensemble structures.
Our work instead targets competitive calibration using inference-time stochasticity
within a single pretrained model.

\subsection{Bayesian and Stochastic Methods}

Bayesian neural networks provide principled uncertainty estimates by placing
distributions over model parameters
\citep{mackay1992practical, neal2012bayesian}.
However, scalable approximations such as variational inference
\citep{blundell2015weight, louizos2017multiplicative} and Laplace methods
\citep{ritter2018scalable} scale poorly to Transformer-sized models.
Bayesian Transformer variants
\citep{xue2021bayesian, fan2020bayesian} further require architectural modification
and retraining, limiting compatibility with pretrained backbones.

Several recent methods estimate uncertainty in weight space.
Stochastic Weight Averaging and its Gaussian extension (SWAG) approximate posterior
distributions during fine-tuning and improve uncertainty quality in NLI tasks
\citep{talman2023uncertainty}. Parameter-efficient Bayesian adaptation methods such as
C-LoRA \citep{rahmati2025c} introduce latent variables within low-rank adapters, but
require Bayesian fine-tuning and task-specific parameter updates. Other approaches
improve uncertainty estimation through post-hoc output heads
\citep{zabolotnyi2025adue} or uncertainty-aware loss reweighting during alignment
\citep{wang2024uncertainty}. These methods primarily influence training dynamics or
output distributions and do not modify internal Transformer computations at inference
time.

Monte Carlo Dropout \citep{gal2016dropout} provides a lightweight alternative by
approximating Bayesian inference through stochastic forward passes. Standard MC
Dropout applies uniform stochasticity across layers and treats uncertainty only as an
output-level signal. Our work builds on MC Dropout, but explicitly uses the resulting
uncertainty to influence attention during inference.

\subsection{Uncertainty-Aware Attention}

Calibration and uncertainty issues are particularly pronounced in Transformer-based
models. Larger Transformers are often increasingly overconfident
\citep{desai2020calibration}, and semantic uncertainty plays a significant role in
language understanding and generation \citep{kuhn2023semantic}. Most existing methods
treat uncertainty as a post-hoc quantity, leaving attention patterns and internal
representations unchanged.

Several approaches inject uncertainty directly into attention mechanisms.
Bayesian Attention Belief Networks \citep{zhang2021bayesian} model attention weights as
latent random variables, while Sparse Gaussian Process Attention
\citep{chen2023calibrating} reformulates self-attention using kernel-based Gaussian
process inference. Although principled, these methods require architectural changes
and variational training, limiting their applicability to pretrained models.
Domain-specific uncertainty-aware Transformers, such as TUnA
\citep{10.1093/bib/bbae359}, introduce uncertainty primarily through output layers
rather than through attention modulation itself.

Recent work explores efficient uncertainty mechanisms inspired by ensembles or
parameter-efficient adaptation.
Hydra Ensembles \citep{gabetni2025ensemblingprunedattentionheads} approximate Deep
Ensembles by fusing multiple attention heads, achieving competitive calibration with
reduced cost. However, uncertainty arises from architectural fusion and remains an
output-level aggregation. Similarly, C-LoRA \citep{rahmati2025c} requires Bayesian
fine-tuning to model uncertainty. Other methods leverage uncertainty for downstream
decision-making, such as selective prediction \citep{vazhentsev2025uncertainty}, but
do not affect internal attention computation.

\paragraph{Positioning of \ours.}
\ours integrates epistemic uncertainty into self-attention at
\emph{inference time}, without modifying pretrained weights, training objectives,
or model architecture.
Unlike post-hoc scalers (e.g., temperature scaling) that adjust confidence only at
the output layer, our method uses uncertainty estimates to \emph{modulate internal
token interactions}, aiming to improve reliability under ambiguity, distribution
shift, and selective decision-making.
Importantly, \ours is \emph{complementary} to post-hoc scaling:
temperature scaling can be applied on top of our uncertainty-aware inference when
the goal is to minimize in-domain calibration error while preserving uncertainty-aware
internal computation.

\section{AI Usage Statement}
\label{app:ai-usage}

We used large language models as auxiliary tools during manuscript preparation. Claude was used for limited coding assistance, including refactoring and debugging support. ChatGPT was used for writing assistance, including grammar checking, spelling correction, and stylistic refinement. All core research contributions, including the scientific ideas, experimental design, implementation, empirical evaluation, analysis, and conclusions, were developed by the authors. All results are based on our own experiments using publicly available datasets and verified code. No AI system was used to generate data, conduct experiments, or draw scientific conclusions.

\end{document}